\RequirePackage{fix-cm}
\documentclass[smallextended]{svjour3}       
\smartqed  
\usepackage{graphicx}
\usepackage{algorithm}
\usepackage[noend]{algorithmic}

\usepackage{float}

\usepackage{lipsum}
\usepackage{amsfonts}
\usepackage{amsmath} 
\usepackage{amsthm}
\usepackage{mathtools}
\usepackage[rgb,dvipsnames]{xcolor}
\usepackage{graphicx}
\graphicspath{{../figures/}}
\usepackage{float}
\usepackage{natbib}
\usepackage{url}
\usepackage{stmaryrd}
\usepackage{mathtools}
\DeclarePairedDelimiter{\ceil}{\lceil}{\rceil}

\newcommand{\frec}{\phi} 
\newcommand{\fout}{\psi}

 \usepackage{relsize}

\usepackage{etoolbox}
\usepackage{xparse}

\renewcommand{\enspace}{}
\newcommand{\R}{\Rbb}
\newcommand{\e}{\ten}
\newcommand{\bmat}[1]{\left[\begin{matrix}
#1
\end{matrix}\right]}

\renewcommand{\vec}[1]{\ensuremath{\mathbf{#1}}}
\newcommand{\vecs}[1]{\ensuremath{\mathbf{\boldsymbol{#1}}}}
\newcommand{\mat}[1]{\ensuremath{\mathbf{#1}}}
\newcommand{\mats}[1]{\ensuremath{\mathbf{\boldsymbol{#1}}}}
\newcommand{\ten}[1]{\mat{\ensuremath{\boldsymbol{\mathcal{#1}}}}}

\newcommand{\ttm}[1]{\times_{#1}}
\newcommand{\ttv}[1]{\bullet_{#1}}
\newcommand{\tenmat}[2]{\ten{#1}_{(#2)}}
\newcommand{\tenmatpar}[2]{(#1)_{(#2)}}
\newcommand{\tenmatgen}[2]{{(#1)}_{\langle\!\langle #2\rangle\!\rangle}}

\newcommand{\TT}[1]{\llbracket #1 \rrbracket}

\usepackage{pgffor}

\foreach \x in {A,...,Z}{%
\expandafter\xdef\csname \x bb\endcsname{\noexpand\ensuremath{\noexpand\mathbb{\x}}}
}

\foreach \x in {A,...,Z}{%
\expandafter\xdef\csname \x cal\endcsname{\noexpand\ensuremath{\noexpand\mathcal{\x}}}
}

\foreach \x in {A,...,Z}{%
\expandafter\xdef\csname \x t\endcsname{\noexpand\ensuremath{\noexpand\ten{\x}}}
}

\foreach \x in {A,...,Z}{%
\expandafter\xdef\csname \x b\endcsname{\noexpand\ensuremath{\noexpand\mat{\x}}}
}

\foreach \x in {a,...,r}{%
\expandafter\xdef\csname \x b\endcsname{\noexpand\ensuremath{\noexpand\vec{\x}}}
}
\foreach \x in {t,...,z}{%
\expandafter\xdef\csname \x b\endcsname{\noexpand\ensuremath{\noexpand\vec{\x}}}
}

\newcommand{\A}{\mat{A}}
\newcommand{\B}{\mat{B}}

\newcommand{\Hten}{\ten{H}}
\newcommand{\T}{\ten{T}}

\newcommand{\G}{\ten{G}}
\newcommand{\Y}{\ten{Y}}
\newcommand{\U}{\mat{U}}
\renewcommand{\H}{\mat{H}}

\newcommand{\X}{\mat{X}}
\newcommand{\I}{\mat{I}}

\newcommand{\M}{\mat{M}}

\newcommand{\V}{\mat{V}}
\renewcommand{\P}{\mat{P}}
\renewcommand{\S}{\mat{S}}
\renewcommand{\v}{\vec{v}}

\newcommand{\x}{\vec{x}}

\newcommand{\y}{\vec{y}}

\newcommand{\Pref}{P}
\newcommand{\Suff}{S}

\newtheorem*{theorem*}{Theorem}
\newtheorem*{corollary*}{Corollary}%
\newtheorem*{proposition*}{Proposition}%
\ifcsdef{theorem}{}{
\newtheorem{theorem}{Theorem}%
\newtheorem{proposition}{Proposition}%
\newtheorem{definition}{Definition}%
}
\newtheorem*{pbm*}{Problem}%
\newtheorem*{algo*}{Algorithm}%

\newcommand{\kron}{\otimes}

\DeclareMathOperator*{\Tr}{Tr} 

\DeclareMathOperator*{\argmin}{arg\,min}
\newcommand{\vectorize}[1]{\mathrm{vec}(#1)}

\DeclareMathOperator*{\rank}{rank}
\newcommand{\norm}[1]{\|#1\|}

\newcommand{\bigo}[1]{\mathcal{O}\left(#1\right)}

\newcommand{\pinv}{^\dagger}
\newcommand{\inv}{^{-1}}
\newcommand{\invtop}{^{-\top}}

\newcommand{\nstates}{n}

\newcommand{\szerosymbol}{\alpha}
\newcommand{\szero}{\vecs{\szerosymbol}}

\newcommand{\sinfsymbol}{\omega}
\newcommand{\sinf}{\vecs{\sinfsymbol}}

\DeclareDocumentCommand{\wa}{  O{A} O{\szero} O{\sinf} }%
{(#2,\{\mat{#1}^\sigma\}_{\sigma\in\Sigma},#3)}
\DeclareDocumentCommand{\waR}{  O{A} O{\Rbb^\nstates} O{\szero} O{\sinf} }%
{(#2,#3,\{\mat{#1}^\sigma\}_{\sigma\in\Sigma},#4)}

\newcommand{\vvsinfsymbol}{\Omega}
\newcommand{\vvsinf}{\mats{\vvsinfsymbol}}
\DeclareDocumentCommand{\vvwa}{  O{A} O{\szero} O{\vvsinf} }%
{(#2,\{\mat{#1}^\sigma\}_{\sigma\in\Sigma},#3)}

\newcommand{\tzerosymbol}{\alpha}
\newcommand{\tzero}{\vecs{\tzerosymbol}}

\newcommand{\tinfsymbol}{\omega}
\newcommand{\tinf}{\vecs{\tinfsymbol}}

\DeclareDocumentCommand{\wta}{ O{T} O{\Rbb^\nstates} O{\tzero} O{\tinf} O{\Fcal}}%
{(#2,#3,\{\ten{#1}^g\}_{g\in #5_{\geq 1}},\{#4^\sigma\}_{\sigma\in #5_0})}

\DeclareDocumentCommand{\trees}{g}{\IfNoValueTF{#1}{\mathfrak{T}}{\mathfrak{T}_{#1}}}
\DeclareDocumentCommand{\contexts}{g}{\IfNoValueTF{#1}{\mathfrak{C}}{\mathfrak{C}_{#1}}}

\newcommand{\gwmprod}{\diamond}

\DeclareDocumentCommand{\gwm}{  O{M} O{\Fbb^\nstates}}{(#2, \{\ten{#1}^x\}_{x\in\Sigma})}
\DeclareDocumentCommand{\gwmcirc}{  O{M} O{\Rbb^\nstates}}{(#2, \{\mat{#1}^\sigma\}_{\sigma\in\Sigma})}
\DeclareDocumentCommand{\dgwm}{ O{M} O{\Fbb^\nstates}}{(#2, \{\ten{#1}^x\}_{x\in\Sigma},\gwmprod)}

\usepackage{pgffor}
\foreach \x in {A,...,Z}{%
\expandafter\xdef\csname \x bb\endcsname{\noexpand\ensuremath{\noexpand\mathbb{\x}}}
}

\foreach \x in {A,...,Z}{%
\expandafter\xdef\csname \x cal\endcsname{\noexpand\ensuremath{\noexpand\mathcal{\x}}}
}

\usepackage{pgffor}
\foreach \x in {A,...,Z}{%
\expandafter\xdef\csname \x ten\endcsname{\noexpand\ensuremath{\noexpand\ten{\x}}}
}

\foreach \x in {A,...,Z}{%
\expandafter\xdef\csname \x mat\endcsname{\noexpand\ensuremath{\noexpand\mat{\x}}}
}

\foreach \x in {A,...,Z}{%
\expandafter\xdef\csname \x vec\endcsname{\noexpand\ensuremath{\noexpand\mat{\x}}}
}

\newcommand{\h}{\vec{h}}
\newcommand{\ie}{i.e.\ }
\newcommand{\eg}{e.g.\ }
\renewcommand{\e}{\vec{e}}

\usepackage[symbol]{footmisc}

\usepackage{caption}
\usepackage{subcaption}
\usepackage[utf8]{inputenc}
\usepackage{tikz,amsmath,siunitx}
\usetikzlibrary{positioning}
\usepackage{graphicx}
\usetikzlibrary{calc}

\usepackage[normalem]{ulem}

\usetikzlibrary{decorations.pathreplacing}
\usetikzlibrary{cd,arrows,shapes,shapes.misc,intersections}

\begin{document}

\title{Connecting Weighted Automata, Tensor Networks and Recurrent Neural Networks through Spectral Learning
}
\subtitle{}

\titlerunning{Weighted Automata, Recurrent Neural Networks and Tensor Networks}        

\author{Tianyu Li \and Doina Precup \and Guillaume Rabusseau 
}
\authorrunning{Li, Precup and Rabusseau} 


\date{Received: date / Accepted: date}

\maketitle

\newcommand{\guillaume}[1]{{\color{blue} #1 }}
\newcommand{\tianyu}[1]{{\color{orange} #1 }}

   
   


\begin{abstract}
In this paper, we present connections between three models used in different research fields: weighted finite automata~(WFA) from formal languages and linguistics, recurrent neural networks used in machine learning, and tensor networks which encompasses a set of optimization techniques for high-order tensors used in quantum physics and numerical analysis. We first present an intrinsic relation between WFA and the tensor train decomposition, a particular form of tensor network. This relation allows us to exhibit a novel low rank structure of the Hankel matrix of a function computed by a WFA and to design an efficient spectral learning algorithm leveraging this structure to scale the algorithm up to very large Hankel matrices.
We then unravel a fundamental connection between WFA and second-order
recurrent neural networks~(2-RNN): in the case of sequences of discrete symbols, WFA and 2-RNN with linear activation
functions are expressively equivalent.  
Leveraging this equivalence result combined with the classical spectral learning algorithm for weighted automata, we introduce the first provable learning algorithm for linear 2-RNN defined over sequences of continuous input vectors.
This algorithm relies on estimating low rank sub-blocks of the  Hankel tensor, from which the parameters of a linear 2-RNN can be provably recovered. 
The performances of the proposed learning algorithm are assessed in a simulation study on both synthetic and real-world data.
\end{abstract}

\section{Introduction}

Many tasks in natural language processing, computational biology, reinforcement learning, and time series analysis rely on learning
with sequential data, \ie estimating functions defined over sequences of observations from training data.  
Weighted  finite automata~(WFA) and recurrent neural networks~(RNN) are two powerful and flexible classes of models which can efficiently represent such functions.
On the one hand, WFA are tractable, they encompass a wide range of machine learning models~(they can for example compute any probability distribution defined by a hidden Markov 
model~(HMM)~\citep{denis2008rational} and can model the transition and observation behavior of
partially observable Markov decision processes~\citep{thon2015links}) and they offer appealing theoretical guarantees. In particular, 
the so-called 
\emph{spectral
methods} for learning HMMs~\citep{hsu2009spectral}, WFA~\citep{bailly2009grammatical,balle2014spectral} and related models~\citep{glaude2016pac,boots2011closing}, 
provide an alternative to Expectation-Maximization based algorithms that is both computationally efficient and 
consistent. 
On the other hand, RNN are 
remarkably expressive models --- they can represent any computable function~\citep{siegelmann1992computational} --- and they have successfully
tackled many practical problems in speech and audio recognition~\citep{graves2013speech,mikolov2011extensions,gers2000learning}, but
their theoretical  analysis is difficult. Even though recent work provides interesting results on their
expressive power~\citep{khrulkov2018expressive,yu2017long} as well as alternative training algorithms coming with learning guarantees~\citep{sedghi2016training},
the theoretical understanding of RNN is still limited.

At the same time, tensor networks are a generalization of tensor decomposition techniques, where complex operations between tensors are represented in a simple diagrammatic notation, allowing one to intuitively represent intricate ways to decompose a high-order tensor into lower-order tensors acting as \emph{building block}. The term \emph{tensor networks} also encompasses a set of optimization techniques to efficiently tackle optimization problems in very high-dimensional spaces, where the optimization variable is represented as a tensor network and the optimization process is carried out with respect to the \emph{building blocks} of the tensor network. As an illustration, such optimization techniques make it possible to efficiently approximate the leading eigen-vectors of matrices of size $2^N\times 2^N$ where $N$ can be as large as 50~\citep{holtz2012alternating}. Tensor networks have emerged in the quantum physics community to model many-body systems~\citep{Orus:arXiv1306.2164,biamonte2017tensor} and have also been used in numerical analysis as a mean to solve high-dimensional differential equations~\citep{oseledets2011tensor,lubich2013dynamical} and to design efficient algorithms for big data analytics~\citep{cichocki2016tensor}. Tensor networks have recently been used in the context of machine learning to compress neural networks~\citep{novikov2015tensorizing,novikov2014putting,ma2019tensorized,yang2017tensor}, to design new approaches and optimization techniques borrowed from the quantum physics literature for supervised and unsupervised learning tasks~\citep{stoudenmire2016supervised,han2018unsupervised,miller2020tensor}, as new theoretical tools to understand the expressiveness of neural networks~\citep{cohen2016expressive,khrulkov2018expressive} and for image completion problems~\citep{yang2017tensor,wang2017efficient} among others.

In this work, we bridge a gap between these three classes of models: weighted automata, tensor networks and recurrent neural networks. We first exhibit an intrinsic relation between the computation of a weighted automata and the tensor train decomposition, a particular form of tensor network~(also known as matrix product states in the quantum physics community). While such a connection has been sporadically noticed previously, we demonstrate how this relation implies a low tensor train structure of the so-called Hankel matrix of a function computed by a WFA. The Hankel matrix of a function is at the core of the spectral learning algorithm for WFA. This algorithm relies on the fact that the~(matrix) rank of the Hankel matrix is directly related to the size of a WFA computing the function it represents. We show that, beyond being low rank, the Hankel matrix of a function computed by a WFA can be seen as a block matrix where each block is a matricization of a tensor with low tensor train rank. Building upon this result, we design an efficient implementation of the spectral learning algorithm that leverages this tensor train structure. When the Hankel matrices needed for the spectral algroithm are given in the tensor train format, the time complexity of the algorithm we propose is exponentially smaller~(w.r.t. the size of the Hankel matrix) than the one of the classical spectral learning algorithm. 

We then unravel a fundamental connection between WFA and second-order RNN~(2-RNN): 
\textit{when considering input sequences of discrete symbols, 2-RNN with linear activation functions and WFA are one and the same}, \ie they are expressively
equivalent and there exists a one-to-one mapping between the two classes~(moreover, this mapping conserves model sizes).  While connections between
finite state machines~(\eg deterministic finite automata)  and recurrent neural networks have been noticed and investigated in the past~(see \eg \citep{giles1992learning,omlin1996constructing}), to the
best of our knowledge this is the first time that such a rigorous equivalence between linear 2-RNN and \emph{weighted} automata is explicitly formalized. 
More precisely, we pinpoint exactly the class of recurrent neural architectures to which weighted automata are equivalent, namely second-order RNN with
linear activation functions.
This result naturally leads to the observation that linear 2-RNN are a natural generalization of WFA~(which take sequences of \emph{discrete} observations as
inputs) to sequences of \emph{continuous vectors}, and raises the question of whether the spectral learning algorithm for WFA can be extended to linear 2-RNN. 
The third contribution of this paper is to show that the answer is in the positive: building upon the classical spectral learning algorithm for WFA~\citep{hsu2009spectral,bailly2009grammatical,balle2014spectral} and its recent extension to vector-valued functions~\citep{rabusseau2017multitask}, \emph{we propose the first provable learning algorithm for second-order RNN with linear activation functions}.
Our learning algorithm relies on estimating  sub-blocks of the so-called Hankel tensor, from which the parameters of a 2-linear RNN can be recovered
using basic linear algebra operations. One of the key technical difficulties in designing this algorithm resides in estimating
these sub-blocks from training data where the inputs are sequences of \emph{continuous} vectors. 
We leverage multilinear properties of linear 2-RNN and the tensor train structure of the Hankel matrix to perform this estimation efficiently using matrix sensing and tensor recovery techniques. In particular, we show that the Hankel matrices needed for learning can be estimated directly in the tensor train format, which allows us to use the efficient spectral learning algorithm in the tensor train format discussed previously.

We validate our theoretical findings in a simulation study on synthetic and real world data where we experimentally compare  
different recovery methods and investigate the robustness of our algorithm to noise.
We also show that refining the estimator returned
by our algorithm using stochastic gradient descent can lead to significant improvements.


\paragraph{Summary of contributions.} 
We present \emph{novel connections between WFA and the tensor train decomposition}~(Section~\ref{sec:TT.struct.hankel}) allowing us to design an \emph{highly efficient implementation of the spectral learning algorithm in the tensor train format}~(Section~\ref{sec:spectral.learning.TT}).
We formalize a \emph{strict equivalence between weighted automata and second-order RNN with linear activation
functions}~(Section~\ref{sec:WFA.2RNN.equivalence}), showing that linear 2-RNN can be seen as a natural extension of (vector-valued) weighted automata for input sequences of \emph{continuous} vectors. We then
propose a \emph{consistent learning algorithm for linear 2-RNN}~(Section~\ref{sec:spec.learn.2RNN}).
The relevance of our contributions can be seen from three perspectives.
First, while learning feed-forward neural networks with linear activation functions is a trivial task (it reduces to linear or reduced-rank regression), this
is not at all the case for recurrent architectures with linear activation functions; to the best of our knowledge, our algorithm is the \emph{first consistent learning algorithm
for the class of functions computed by linear second-order recurrent networks}. Second, from the perspective of learning weighted automata, we propose a  natural extension of WFA to continuous inputs and \emph{our learning algorithm addresses the long-standing limitation of the spectral learning method to discrete inputs}. Lastly, by connecting the spectral learning algorithm for WFA to recurrent neural networks on one side, and tensor networks on the other, our work opens the door to leveraging highly efficient optimization techniques for large scale tensor problems used in the quantum physics community for designing new learning learning algorithms for both linear and non-linear sequential models, as well as offering new tools for the theoretical analysis of these models.

\paragraph{Related work.}
Combining the spectral learning algorithm for WFA with matrix completion techniques~(a problem which is closely related to matrix sensing) has
been theoretically investigated in~\citep{balle2012spectral}. An extension of probabilistic transducers to continuous inputs~(along with a spectral learning algorithm) has been proposed in~\citep{recasens2013spectral}. The model considered in this work is closely related to the continuous extension of WFA we consider here but the learning algorithm proposed in~\citep{recasens2013spectral} is designed for~(and limited to) stochastic transducers, whereas we consider arbitrary functions computed by linear 2-RNN. 
The connections between tensors  and RNN have been previously leveraged to study the expressive power of RNN in~\citep{khrulkov2018expressive}
and to achieve model compression in~\citep{yu2017long,yang2017tensor,tjandra2017compressing}. 
Exploring relationships between RNN and automata has recently received a renewed interest~\citep{peng2018rational,chen2018recurrent,li2018nonlinear,merrill2020formal}. In particular, such connections have been explored for interpretability purposes~\citep{weiss2018extracting,ayache2018explaining} and the ability of RNN to learn classes of formal languages
has been investigated in~\citep{avcu2017subregular}.  Connections between the tensor train decomposition and WFA have been previously noticed in~\citep{critch2013algebraic,critch2014algebraic,rabusseau2016thesis}. However, to the best of our knowledge, this is the first time that the tensor-train structure of the Hankel matrix of a function computed by a WFA is noticed and leveraged to design an efficient spectral learning algorithm for WFA. Other approaches have been proposed to scale the spectral learning algorithm to large datasets, notably by identifying a small basis of informative prefixes and suffixes to build the Hankel matrices~\citep{quattoni2017maximum}.
The predictive state RNN model introduced in~\citep{downey2017predictive} is closely related to 2-RNN and the authors propose
to use the spectral learning algorithm for predictive state representations to initialize a gradient based algorithm; their approach however comes without
theoretical guarantees. Lastly, a provable algorithm for RNN relying on the tensor method of moments has been proposed in~\citep{sedghi2016training} but
it is limited to first-order RNN with quadratic activation functions~(which do not encompass linear 2-RNN).

\section{Preliminaries}
In this section, we first present basic notions of tensor algebra and tensor networks before introducing  second-order recurrent neural 
network, 
weighted finite automata and the spectral learning algorithm.
We start by introducing some notations.
For any integer $k$ we use $[k]$ to denote the set of integers from $1$ to $k$. We use
$\lceil l \rceil$ to denote the smallest integer greater or equal to $l$.
For any set $\Suff$, we denote by $\Suff^*=\bigcup_{k\in\Nbb}\Suff^k$ the set of all
finite-length sequences of elements of $\Suff$~(in particular, 
$\Sigma^*$ will denote the set of strings on a finite alphabet $\Sigma$). 
We use lower case bold letters  for vectors (\eg $\vec{v} \in \Rbb^{d_1}$),
upper case bold letters for matrices (\eg $\M \in \Rbb^{d_1 \times d_2}$) and
bold calligraphic letters for higher order tensors (\eg $\T \in \Rbb^{d_1
\times d_2 \times d_3}$). We use $\e_i$ to denote the $i$th canonical basis 
vector of $\R^d$~(where the dimension $d$ will always appear clearly from context).
The $d\times d$ identity matrix will be written as $\I_d$.
The $i$th row (resp. column) of a matrix $\M$ will be denoted by
$\M_{i,:}$ (resp. $\M_{:,i}$). This notation is extended to
slices of a tensor in the straightforward way.
If $\vec{v} \in \Rbb^{d_1}$ and $\vec{v}' \in \Rbb^{d_2}$, we use $\vec{v} \kron \vec{v}' \in \Rbb^{d_1
\cdot d_2}$ to denote the Kronecker product between vectors, and its
straightforward extension to matrices and tensors.
Given a matrix $\M \in \Rbb^{d_1 \times d_2}$, we use $\vectorize{\M} \in \Rbb^{d_1
\cdot d_2}$ to denote the column vector obtained by concatenating the columns of
$\M$. The inverse of $\M$ is denoted by $\M\inv$, its Moore-Penrose pseudo-inverse
by $\M\pinv$, and the transpose of its inverse by $\M\invtop$; the Frobenius norm
is denoted by $\norm{\M}_F$ and the nuclear norm by $\norm{\M}_*$.

\subsection{Tensors and Tensor Networks}
We first recall basic definitions of tensor algebra; more details can be found
in~\citep{Kolda09}. 
A \emph{tensor} $\T\in \Rbb^{d_1\times\cdots \times d_p}$ can simply be seen
as a multidimensional array $(\T_{i_1,\cdots,i_p}\ : \ i_n\in [d_n], n\in [p])$. The
\emph{mode-$n$} fibers of $\T$ are the vectors obtained by fixing all
indices except  the $n$th one, \eg $\T_{:,i_2,\cdots,i_p}\in\Rbb^{d_1}$.
The \emph{$n$th mode matricization} of $\T$ is the matrix having the
mode-$n$ fibers of $\T$ for columns and is denoted by
$\tenmat{T}{n}\in \Rbb^{d_n\times d_1\cdots d_{n-1}d_{n+1}\cdots d_p}$.
The vectorization of a tensor is defined by $\vectorize{\T}=\vectorize{\tenmat{T}{1}}$.
In the following $\T$ always denotes a tensor of size $d_1\times\cdots \times d_p$.

The \emph{mode-$n$ matrix product} of the tensor $\T$ and a matrix
$\X\in\Rbb^{m\times d_n}$ is a tensor  denoted by $\T\ttm{n}\X$. It is 
of size $d_1\times\cdots \times d_{n-1}\times m \times d_{n+1}\times
\cdots \times d_p$ and is defined by the relation 
$\Y = \T\ttm{n}\X \Leftrightarrow \tenmat{Y}{n} = \X\tenmat{T}{n}$.
The \emph{mode-$n$ vector product} of the tensor $\T$ and a vector
$\vec{v}\in\Rbb^{d_n}$ is a tensor defined by $\T\ttv{n}\vec{v} = \T\ttm{n}\vec{v}^\top
\in \Rbb^{d_1\times\cdots \times d_{n-1}\times d_{n+1}\times
\cdots \times d_p}$.
%
%
It is easy to check that the $n$-mode product satisfies $(\T\ttm{n}\mat{A})\ttm{n}\mat{B} = \T\ttm{n}\mat{BA}$
where we assume compatible dimensions of the tensor $\T$ and
the matrices $\A$ and $\B$.

\emph{Tensor network diagrams} allow one to represent complex operations on tensors in a graphical and intuitive way. A tensor network is simply a graph where nodes represent tensors, and edges represent contractions between tensor modes, i.e. a summation over an index shared by two tensors. In a tensor network, the arity of a vertex~(i.e. the number of \emph{legs} of a node) corresponds to the order of the tensor: a node with one leg represents a vector, a node with two legs represents a matrix, and a node with three legs represents  a 3rd order tensor~(see Figure~\ref{fig:tensor.network.intro}). We will sometimes add indices to legs of a tensor network to refer to its components or sub-tensors.  For example, the following tensor networks represent a matrix $\A\in\R^{m\times n}$, the $i$th row of $\A$ and the component $\A_{i,j}$ respectively: 
\begin{center}
\begin{tikzpicture}
\tikzset{tensor/.style = {minimum size = 0.5cm,shape = circle,thick,draw=black,fill=blue!60!green!40!white,inner sep = 0pt}, edge/.style   = {thick,line width=.4mm},every loop/.style={}}
\def\x{0}
\def\y{0}
\node[tensor] (A) at (\x,\y) {$\Ab$};
\draw[edge] (A) -- (\x+0.75,\y); 
\draw[edge] (A) -- (\x-0.75,\y); 
\node[draw=none] () at (\x-0.5,\y+0.2) {\textcolor{gray}{$m$}};
\node[draw=none] () at (\x+0.5,\y+0.2) {\textcolor{gray}{$n$}};

\def\x{3.5}
\def\y{0}
\node[tensor] (A) at (\x,\y) {$\Ab$};
\draw[edge] (A) -- (\x+0.75,\y); 
\draw[edge] (A) -- (\x-0.75,\y); 
\node[draw=none] () at (\x-0.85,\y) {$i$};

\def\x{7}
\def\y{0}
\node[tensor] (A) at (\x,\y) {$\Ab$};
\draw[edge] (A) -- (\x+0.75,\y); 
\draw[edge] (A) -- (\x-0.75,\y); 
\node[draw=none] () at (\x-0.85,\y) {$i$};
\node[draw=none] () at (\x+0.85,\y) {$j$};
\end{tikzpicture}
\end{center}

\begin{figure}
    \centering
    \begin{tikzpicture}
\tikzset{tensor/.style = {minimum size = 0.4cm,shape = circle,thick,draw=black,fill=blue!60!green!40!white,inner sep = 1pt}, edge/.style   = {thick,line width=.4mm}}

\def\x{0}
\node[tensor] (v) at (\x,0) {$\vb$};
\draw[edge] (v) -- (\x,-0.6);
\node[draw=none] () at (\x+0.15,-0.4) {\textcolor{gray}{$d$}};

\def\x{3}
\node[tensor] (M) at (\x,0) {$\Mb$};
\draw[edge] (M) -- (\x-0.6,0);
\draw[edge] (M) -- (\x+0.6,0);
\node[draw=none] () at (\x-0.4,0.15) {\textcolor{gray}{$m$}};
\node[draw=none] () at (\x+0.4,0.15) {\textcolor{gray}{$n$}};

\def\x{6}
\node[tensor] (T) at (\x,0) {$\Tt$};
\draw[edge] (T) -- (\x-0.6,0);
\draw[edge] (T) -- (\x+0.6,0);
\draw[edge] (T) -- (\x,-0.6);
\node[draw=none] () at (\x-0.4,0.2) {\textcolor{gray}{$d_1$}};
\node[draw=none] () at (\x+0.2,-0.4) {\textcolor{gray}{$d_2$}};
\node[draw=none] () at (\x+0.4,0.2) {\textcolor{gray}{$d_3$}};
\end{tikzpicture}
    \caption{ \ Tensor network representation of a vector $\vb\in\Rbb^d$, a matrix $\Mb\in\Rbb^{m\times n}$ and a 
    tensor $\Tt\in\Rbb^{d_1\times d_2\times d_3}$. The gray labels over the edges indicate the dimensions of the corresponding modes of the tensors (such labels will only be sporadically displayed when necessary to avoid confusion).}
    \label{fig:tensor.network.intro}
\end{figure}
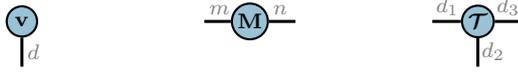

Connecting two legs in a tensor network represents a contraction over the corresponding indices. Consider the following simple tensor network with two nodes:
\begin{center}
\begin{tikzpicture}
\tikzset{tensor/.style = {minimum size = 0.5cm,shape = circle,thick,draw=black,fill=blue!60!green!40!white,inner sep = 0pt}, edge/.style   = {thick,line width=.4mm},every loop/.style={}}
\def\x{0}
\def\y{0}
\node[tensor] (A) at (\x,\y) {$\Ab$};
\node[tensor] (x) at (\x+1,\y) {$\xb$};
\draw[edge] (A) -- (x); 
\draw[edge] (A) -- (\x-0.75,\y); 
\node[draw=none] () at (\x-0.5,\y+0.2) {\textcolor{gray}{$m$}};
\node[draw=none] () at (\x+0.5,\y+0.2) {\textcolor{gray}{$n$}};
\end{tikzpicture}
\end{center}
The first node represents a matrix $\Ab\in\R^{m\times n}$ and the second one a vector $\x\in\R^{n}$. Since this tensor network has one dangling leg~(i.e. an edge which is not connected to any other node), it represents a vector. The edge between the second leg of $\Ab$ and the leg of $\x$ corresponds to a summation over the second mode of $\Ab$ and the first mode of $\x$,. Hence, the resulting tensor network represents the classical matrix-product, which can be seen by calculating the $i$th component of this tensor network:

\begin{center}
\begin{tikzpicture}
\tikzset{tensor/.style = {minimum size = 0.4cm,shape = circle,thick,draw=black,fill=blue!60!green!40!white,inner sep = 0pt}, edge/.style   = {thick,line width=.4mm},every loop/.style={}}
\def\x{0}
\def\y{0}
\node[tensor] (A) at (\x,\y) {$\Ab$};
\node[tensor] (x) at (\x+0.75,\y) {$\xb$};
\draw[edge] (A) -- (x); 
\draw[edge] (A) -- (\x-0.5,\y); 
\node[draw=none] () at (\x-0.6,\y) {$i$};
\node[draw=none] at (\x+3,\y) {$=\sum_j\Ab_{ij}\xb_j = (\Ab\xb)_{i}$};
\end{tikzpicture}
\end{center}
Other examples of tensor network representations of common operations on vectors, matrices and tensors can be found in Figure~\ref{fig:tensor.network.commonoperations}.

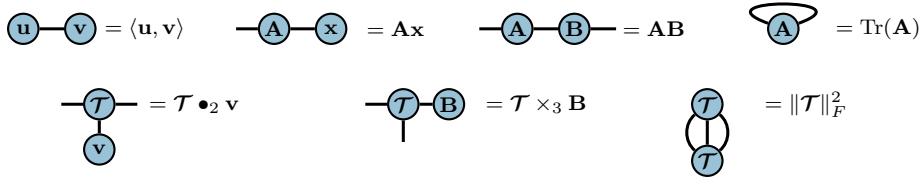
\begin{figure}
    \centering
    \begin{tikzpicture}
\tikzset{tensor/.style = {minimum size = 0.4cm,shape = circle,thick,draw=black,fill=blue!60!green!40!white,inner sep = 0pt}, edge/.style   = {thick,line width=.4mm},every loop/.style={}}

\def\x{-1}
\def\y{0}
\node[tensor] (v) at (\x+0.75,\y) {$\vb$};
\node[tensor] (u) at (\x,\y) {$\ub$};
\draw[edge] (u) -- (v); 
\node[draw=none] at (\x+1.6,\y) {$=\langle \ub,\vb\rangle$};

\def\x{2.3}
\def\y{0}
\node[tensor] (A) at (\x,\y) {$\Ab$};
\node[tensor] (x) at (\x+0.75,\y) {$\xb$};
\draw[edge] (A) -- (x); 
\draw[edge] (A) -- (\x-0.5,\y); 
\node[draw=none] at (\x+1.6,\y) {$=\Ab\xb$};

\def\x{5.5}
\def\y{0}
\node[tensor] (A2) at (\x,\y) {$\Ab$};
\node[tensor] (x2) at (\x+0.75,\y) {$\Bb$};
\draw[edge] (A2) -- (x2); 
\draw[edge] (A2) -- (\x-0.5,\y); 
\draw[edge] (x2) -- (\x+1.3,\y); 
\node[draw=none] at (\x+1.8,\y) {$=\Ab\Bb$};

\def\x{9}
\def\y{0}
\node[tensor] (A) at (\x,\y) {$\Ab$};
\path  (A)  edge [loop above,edge,in=20,out=160,distance=1cm]  (A);
\node[draw=none] at (\x+1.25,\y) {$=\Tr(\Ab)$};

\def\x{0}
\def\y{-1}
\node[tensor] (T) at (\x,\y) {$\Tt$};
\node[tensor] (v2) at (\x,\y-0.6) {$\vb$};
\draw[edge] (T) -- (\x-0.5,\y);
\draw[edge] (T) -- (\x+0.5,\y);
\draw[edge] (T) -- (v2);
\node[draw=none] at (\x+1.25,\y) {$=\Tt\ttv{2}\vb$};

\def\x{4}
\def\y{-1}
\node[tensor] (T2) at (\x,\y) {$\Tt$};
\node[tensor] (B) at (\x+0.6,\y) {$\Bb$};
\draw[edge] (T2) -- (\x-0.5,\y);
\draw[edge] (T2) -- (B);
\draw[edge] (T2) -- (\x,\y-0.5);
\node[draw=none] at (\x+1.75,\y) {$=\Tt\ttm{3}\Bb$};

\def\x{8}
\def\y{-1}
\node[tensor] (T3) at (\x,\y) {$\Tt$};
\node[tensor] (T4) at (\x,\y-0.75) {$\Tt$};
\draw[edge] (T3) -- (T4);
\path  (T3)  edge [bend left=50,edge]  (T4);
\path  (T4)  edge [bend left=50,edge]  (T3);
\node[draw=none] at (\x+1.3,\y) {$=\|\Tt\|_F^2$};

\end{tikzpicture}

    \caption{ \ Tensor network representation of common operation on vectors, matrices and tensors.}
    \label{fig:tensor.network.commonoperations}
\end{figure}

Given strictly positive integers $n_1,\cdots, n_k$ satisfying
$\sum_i n_i = p$, we use the notation $\tenmatgen{\T}{n_1,n_2,\cdots,n_k}$ to denote the $k$th order tensor 
obtained by reshaping $\T\in\R^{d_1\times \cdots\times d_p}$ into a tensor\footnote{Note that the specific ordering used to perform matricization, vectorization
and such a reshaping is not relevant as long as it is consistent across all operations.} of size 
$$(\prod_{i_1=1}^{n_1} d_{i_1}) \times (\prod_{i_2=1}^{n_2} d_{n_1 + i_2}) \times \cdots \times (\prod_{i_k=1}^{n_k} d_{n_1+\cdots+n_{k-1} + i_k}).$$
For example, for a tensor $\At$ of size $2\times 3\times 4\times 5\times 6$, the 3rd order tensor $\tenmatgen{\At}{2,1,2}$ is obtained by grouping the first two modes and the last two modes respectively, to obtain a tensor of size $6\times 4 \times 30$. This reshaping operation is related to vectorization and matricization by the following relations: $\tenmatgen{\T}{p} = \vectorize{\T}$ and $\tenmatgen{\T}{1,p-1} = \tenmat{\T}{1}$.

A rank $R$ \emph{tensor train (TT) decomposition}~\citep{oseledets2011tensor} of a tensor 
$\T\in\R^{d_1\times \cdots\times d_p}$ consists in factorizing $\T$ into the product of $p$ core tensors
$\G_1\in\R^{d_1\times R},\G_2\in\R^{R\times d_2\times R},
\cdots, \G_{p-1}\in\R^{R\times d_{p-1} \times R},
\G_p \in \R^{R\times d_p}$, and is defined\footnote{The classical definition of the TT-decomposition allows the rank $R$ to be different
for each mode, but this definition is sufficient for the purpose of this paper.} by
\begin{equation}
\label{eq:TT}
\T_{i_1,\cdots,i_p} =  
(\G_1)_{i_1,:}(\G_2)_{:,i_2,:}\cdots 
 (\G_{p-1})_{:,i_{p-1},:}(\G_p)_{:,i_p}
\end{equation}
for all indices $i_1\in[d_1],\cdots,i_p\in[d_p]$~(here $(\G_1)_{i_1,:}$ is a row vector, $(\G_2)_{:,i_2,:}$ is an $R\times R$ matrix, etc.).  We will use the notation $\T = \TT{\G_1,\cdots,\G_p}$
to denote such a decomposition. A tensor network representation of this decomposition is shown in Figure~\ref{fig:tensor.train}. The name of this decomposition comes from the fact that the tensor $\Tt$ is decomposed into a train of lower-order tensors. This decomposition is also known in the quantum physics community as \emph{Matrix Product States}~\citep{Orus:arXiv1306.2164,schollwock2011density}, where this denomination comes from the fact that each entry of $\Tt$ is given by a product of matrices, see Eq.~\eqref{eq:TT}.

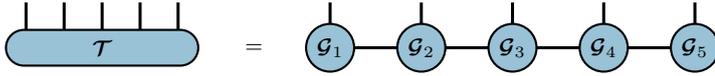
\begin{figure} 
\centering
\begin{tikzpicture}
    \tikzset{tensor/.style = {minimum size = 0.4cm,shape = circle,thick,draw=black,fill=blue!60!green!40!white,inner sep = 0pt}, edge/.style   = {thick,line width=.4mm},every loop/.style={}}
  \pgfmathsetmacro{\start}{3}
  \pgfmathsetmacro{\inc}{1.2}
    \draw[edge] (\start,0)  -- (\start+4*\inc,0)  {};
  \foreach [count=\jj] \ii in {0,...,4} {
  \pgfmathsetmacro{\x}{\start+\ii*\inc}
   \draw[edge] (\x,0.2)  -- (\x,0.6)  {};
   \node[tensor, inner sep=0pt, minimum size=18pt] (1\x) at (\x ,0) {$\Gt_{\jj}$};
  }
 
  \pgfmathsetmacro{\start}{-1}
  \pgfmathsetmacro{\inc}{0.5}
  \foreach  \ii in {0,...,4} {
  \pgfmathsetmacro{\x}{\start+\ii*\inc}
   \draw[edge] (\x,0.2)  -- (\x,0.6)  {};
  }
 \node[tensor, rounded rectangle, minimum width=80pt, inner sep=4pt] (T) at (0,0) {$\Tt$};
 \node[draw=none] (eq) at (2,0) {$ = $};
\end{tikzpicture}
\caption{\ Tensor network representation of a tensor train decomposition.}\label{fig:tensor.train}
\end{figure}

 While the problem of finding the best approximation of  TT-rank $R$
of a given tensor is NP-hard~\citep{hillar2013most}, 
a quasi-optimal SVD based compression algorithm~(TT-SVD) has been proposed 
in~\citep{oseledets2011tensor}.
It is worth mentioning that the TT decomposition is invariant under change of basis: 
for any invertible matrix $\M$ and any core tensors $\G_1,\G_2,\cdots,\G_p$, we have
\begin{equation*}\label{eq:TT.change.of.basis}
\TT{\G_1,\cdots,\G_p} 
=
\TT{\G_1\ttm{2}\M\invtop,\G_2\ttm{1}\M\ttm{3}\M\invtop,\cdots,
\G_{p-1}\ttm{1}\M\ttm{3}\M\invtop,\G_p\ttm{1}\M}
.
\end{equation*}
This relation appears clearly using tensor network diagrams, e.g. with $p=4$ we have\footnote{Note that the colors of the nodes do not bear any meaning and are simply used as visual clues to help parse the diagrams.}:
\begin{center}
    \input{tikz/TT_basis_invariance}
\end{center}

\subsection{Weighted Automata and Spectral Learning}\label{sec:wfa.and.speclearn}
\emph{Vector-valued weighted finite automata}~(vv-WFA) have
been introduced in~\citep{rabusseau2017multitask} as a natural generalization of weighted automata from scalar-valued functions
to vector-valued ones. 

\begin{definition}

A $p$-dimensional vv-WFA with $n$ states is a tuple $A=\vvwa$
where $\szero\in\R^n$ is the initial weights vector, $\vvsinf\in\R^{p\times n}$ is the matrix of final weights, 
and $\A^\sigma\in\R^{n\times n}$ is the transition matrix for each symbol $\sigma$ in a finite alphabet $\Sigma$.
A vv-WFA $A$ computes a function $f_A:\Sigma^*\to\R^p$ defined by 
$$f_A(x) =\vvsinf (\A^{x_1}\A^{x_2}\cdots\A^{x_k})^\top\szero $$
for each word $x=x_1x_2\cdots x_k\in\Sigma^*$. 
\end{definition}

We call a vv-WFA \emph{minimal} if its number of states
is minimal, that is, any vv-WFA computing the same function as at least as many states as the minimal vv-WFA. Given a function $f:\Sigma^*\to \Rbb^p$, we denote by $\rank(f)$ the number of states of a minimal vv-WFA computing $f$~(which is
set to $\infty$ if $f$ cannot be computed by a vv-WFA).

The \emph{spectral learning algorithm} is a consistent learning algorithm for weighted finite automata. It has been introduced concurrently in~\citep{hsu2009spectral} and \citep{bailly2009grammatical}~(see~\citep{balle2014spectral} for a comprehensive presentation of the algorithm). This algorithm relies on a fundamental object: \emph{the Hankel matrix}. Given a function $f:\Sigma^*\to\Rbb$, its Hankel matrix $\H\in\Rbb^{\Sigma^*\times \Sigma^*}$ is the bi-infinite matrix defined by
$$
\H_{u,v} = f(uv)\ \ \ \ \text{for all }u,v\in\Sigma^*
$$
where $uv$ denotes the concatenation of the prefix $u$ and the suffix $v$. The striking relation between the Hankel matrix and the rank of a function $f$ has been well known in the formal language community~\citep{fliess1974matrices,carlyle1971realizations} and is at the heart of the spectral learning algorithm. This relation states that the rank of the Hankel matrix of a function $f$ exactly coincides with the rank of $f$, i.e. the number of states of the smallest WFA computing $f$. In particular, the rank of the Hankel matrix of $f$ is finite if and only if $f$ can be computed by a weighted automaton. An example of a function which cannot be computed by a WFA is the indicator function of the language $a^nb^n$~(on the alphabet $\Sigma=\{a,b\}$):
\begin{equation}
f(x)
=
\begin{cases}
1& \text{ if } x=a^nb^n\text{ for some integer }n\\
0&\text{otherwise.}
\end{cases}
\end{equation}

The spectral learning algorithm was naturally extended to vector-valued WFA in~\citep{rabusseau2017multitask}, where the Hankel matrix is replaced by the \emph{Hankel tensor} $\Ht\in\Rbb^{\Sigma^*\times\Sigma^*\times p}$ of a vector-valued function $f:\Sigma^*\to\R^p$, which is defined by
$$
\Hten_{u,v,:} = f(uv)\ \ \ \ \text{ for all }u,v\in\Sigma^*.
$$
The relation between the rank of the Hankel matrix and the function $f$ naturally carries over to the vector-valued case and is given in the following theorem.
\begin{theorem}[\cite{rabusseau2017multitask}]
\label{thm:fliess-vvWFA}
Let $f:\Sigma^*\to\R^d$ and let $\Hten$ be its Hankel tensor. Then $\rank(f) = \rank(\tenmat{H}{1})$.
\end{theorem}

The vv-WFA learning algorithm leverages the fact that the proof of this theorem is constructive: one can recover a vv-WFA computing $f$
from any low rank factorization of $\tenmat{H}{1}$. In practice, a finite sub-block $\Hten_{\Pref,\Suff} \in \R^{\Pref\times \Suff\times p}$ of the Hankel tensor  is used
to recover the vv-WFA, where $\Pref,\Suff\subset\Sigma^*$ are finite sets of prefixes and suffixes forming a \emph{complete basis} for $f$, \ie such that
$\rank(\tenmatpar{\Hten_{\Pref,\Suff}}{1}) = \rank(\tenmat{H}{1})$. Indeed, one can show that Theorem~\ref{thm:fliess-vvWFA} still holds when replacing the Hankel tensor by such a sub-block  $\Hten_{\Pref,\Suff}$. The spectral learning algorithm then consists of the following steps:
\begin{enumerate}
    \item Choose a target rank $n$ and a set of prefixes and suffixes $\Pref,\Suff\subset\Sigma^*$.
    \item Estimate the following sub-block of the Hankel tensor from data:
    \begin{itemize}
        \item[$\bullet$] $\Ht_{\Pref,\Suff}\in\R^{\Pref\times\Suff\times p}$ defined by $(\Ht_{\Pref,\Suff})_{u,v,:}=f(uv)$ for all $u\in\Pref,v\in\Suff$.
        \item[$\bullet$] $\H_{\Pref}\in\R^{\Pref\times p}$ defined by $(\H_{\Pref})_{u,:}=f(u)$ for all $u\in\Pref$.
        \item[$\bullet$] $\H_{\Suff}\in\R^{\Suff\times p}$ defined by $(\H_{\Suff})_{v,:}=f(v)$ for all $v\in\Suff$.
        \item[$\bullet$] $\Ht^\sigma_{\Pref,\Suff}\in\R^{\Pref\times\Suff\times p}$ for each $\sigma\in\Sigma$ defined by $(\Ht^\sigma_{\Pref,\Suff})_{u,v,:}=f(u\sigma v)$ for all $u\in\Pref,v\in\Suff$.
    \end{itemize}
    \item Obtain a (approximate) low rank factorization of the Hankel tensor~(using e.g. truncated SVD) $$\tenmatpar{\Hten_{\Pref,\Suff}}{1}\simeq \P\tenmat{S}{1}$$ 
    where $\P\in\R^{\Pref\times n}$ and $\ten{S}\in \R^{n \times \Suff \times p}$.
    \item Compute the parameters of the learned vv-WFA using the relations
    \begin{align*}
        \szero^\top &= \vectorize{\H_\Suff}^\top (\tenmat{\St}{1})\pinv\\
        \vvsinf &=\P\inv\H_\Pref\\
        \A^\sigma &=  \P\pinv\Hten^\sigma_{(1)}(\tenmat{S}{1})\pinv  \ \text{ for each }\sigma\in\Sigma.
    \end{align*}
\end{enumerate}
This learning algorithm is \emph{consistent}: in the limit of infinite training data~(i.e. the Hankel sub-blocks are exactly estimated from data), this algorithm is guaranteed to return a WFA that computes the target function $f$ if $\Pref$ and $\Suff$ form a complete basis. That is, the algorithm is consistent if the rank of the sub-block $\tenmatpar{\Hten_{\Pref,\Suff}}{1}$ is equal to the rank of the full Hankel tensor, i.e. $\rank(\tenmatpar{\Hten_{\Pref,\Suff}}{1}) = \rank(\tenmat{H}{1})$. More details can be found
in~\citep{balle2014spectral} for WFA and in~\citep{rabusseau2017multitask} for vv-WFA. Using tensor network diagrams, steps 3) and 4) of the spectral learning algorithm can be represented as follows:
\begin{center}
    
\begin{tikzpicture}

\tikzset{tensor/.style = {minimum size = 0.4cm,shape = circle,thick,draw=black,fill=blue!60!green!40!white,inner sep = 0pt}, edge/.style   = {thick,line width=.4mm},every loop/.style={}}

\def\x{0}
\def\y{0}
\node[draw=none] at (\x,\y+0.7) {$\Ht_{\Pref,\Suff}$};
\node[tensor] (G0) at (\x,\y) {};
\draw[edge] (G0) -- (\x,\y-0.6);
\draw[edge] (G0) -- (\x+0.75,\y);
\draw[edge] (G0) -- (\x-0.75,\y);
\node[draw=none] () at (\x-0.5,\y+0.2) {\textcolor{gray}{$\Pref$}};
\node[draw=none] () at (\x+0.5,\y+0.2) {\textcolor{gray}{$p$}};
\node[draw=none] () at (\x+0.2,\y-0.4) {\textcolor{gray}{$\Suff$}};

\node[draw=none] () at (\x+1.2,\y) {$\simeq$};

\def\x{2.5}
\def\y{0}
\node[draw=none] at (\x,\y+0.7) {$\P$};
\node[tensor,fill=green!60!red!40!white] (P) at (\x,\y) {};
\draw[edge] (P) -- (\x+0.75,\y);
\draw[edge] (P) -- (\x-0.75,\y);
\node[draw=none] () at (\x-0.5,\y+0.2) {\textcolor{gray}{$\Pref$}};
\node[draw=none] () at (\x+0.5,\y+0.2) {\textcolor{gray}{$n$}};

\def\x{3.5}
\def\y{0}
\node[draw=none] at (\x,\y+0.7) {$\St$};
\node[tensor,fill=blue!60!red!60!white] (S) at (\x,\y) {};
\draw[edge] (S) -- (\x+0.75,\y);
\draw[edge] (S) -- (\x-0.75,\y);
\draw[edge] (S) -- (\x,\y-0.6);
\node[draw=none] () at (\x+0.5,\y+0.2) {\textcolor{gray}{$p$}};
\node[draw=none] () at (\x+0.2,\y-0.4) {\textcolor{gray}{$\Suff$}};

\end{tikzpicture} 

\bigskip

\begin{tikzpicture}

\tikzset{tensor/.style = {minimum size = 0.4cm,shape = circle,thick,draw=black,fill=blue!60!green!40!white,inner sep = 0pt}, edge/.style   = {thick,line width=.4mm},every loop/.style={}}

\def\x{0}
\def\y{0}
\node[draw=none] at (\x,\y+0.5) {$\szero$};
\node[tensor,fill=lightgray] (G0) at (\x,\y) {};
\draw[edge] (G0) -- (\x+0.75,\y);
\node[draw=none] () at (\x+0.5,\y+0.2) {\textcolor{gray}{$n$}};

\node[draw=none] () at (\x+1.2,\y) {$=$};

\def\x{2}
\def\y{0}

\node[draw=none] at (\x,\y+0.6) {$\H_{\Suff}$};
\node[tensor,fill=blue!60!green!40!white] (G0) at (\x,\y) {};
\draw[edge] (G0) -- (\x+0.75,\y);
\node[draw=none] () at (\x+0.8,\y+0.2) {\textcolor{gray}{$p$}};

\def\x{3.5}
\def\y{0}
\node[draw=none] at (\x,\y+0.6) {$(\tenmat{\St}{1})\pinv$};
\node[tensor,fill=blue!60!red!60!white] (S) at (\x,\y) {$\dagger$};
\draw[edge] (S) -- (\x+0.75,\y);
\draw[edge] (S) -- (\x-0.75,\y);
\path  (S)  edge [bend left,edge,in=90,out=90,distance=0.5cm]  (G0);
\node[draw=none] () at (\x+0.5,\y+0.2) {\textcolor{gray}{$n$}};
\node[draw=none] () at (\x-0.75,\y-0.8) {\textcolor{gray}{$\Suff$}};

\def\x{6}
\def\y{0}
\node[draw=none] at (\x,\y+0.5) {$\vvsinf$};
\node[tensor,fill=lightgray] (G1) at (\x,\y) {};
\draw[edge] (G1) -- (\x+0.75,\y);
\draw[edge] (G1) -- (\x-0.75,\y);
\node[draw=none] () at (\x+0.5,\y+0.2) {\textcolor{gray}{$p$}};
\node[draw=none] () at (\x-0.5,\y+0.2) {\textcolor{gray}{$n$}};

\node[draw=none] () at (\x+1.2,\y) {$=$};

\def\x{8.5}
\def\y{0}
\node[draw=none] at (\x,\y+0.6) {$\P\pinv$};
\node[tensor,fill=green!60!red!40!white] (P) at (\x,\y) {$\dagger$};
\draw[edge] (P) -- (\x+0.75,\y);
\draw[edge] (P) -- (\x-0.75,\y);
\node[draw=none] () at (\x-0.5,\y+0.2) {\textcolor{gray}{$n$}};
\node[draw=none] () at (\x+0.8,\y+0.2) {\textcolor{gray}{$\Pref$}};

\def\x{10}
\def\y{0}

\node[draw=none] at (\x,\y+0.6) {$\H_{\Pref}$};
\node[tensor,fill=blue!60!green!40!white] (G0) at (\x,\y) {};
\draw[edge] (G0) -- (\x+0.75,\y);
\draw[edge] (G0) -- (\x-0.75,\y);
\node[draw=none] () at (\x+0.6,\y+0.2) {\textcolor{gray}{$p$}};

\end{tikzpicture} 

\vspace{-0.1cm}

\begin{tikzpicture}

\tikzset{tensor/.style = {minimum size = 0.4cm,shape = circle,thick,draw=black,fill=blue!60!green!40!white,inner sep = 0pt}, edge/.style   = {thick,line width=.4mm},every loop/.style={}}

\def\x{0}
\def\y{0}
\node[draw=none] at (\x,\y+0.5) {$\A^\sigma$};
\node[tensor,fill=lightgray] (G1) at (\x,\y) {};
\draw[edge] (G1) -- (\x+0.75,\y);
\draw[edge] (G1) -- (\x-0.75,\y);
\node[draw=none] () at (\x+0.5,\y+0.2) {\textcolor{gray}{$n$}};
\node[draw=none] () at (\x-0.5,\y+0.2) {\textcolor{gray}{$n$}};

\node[draw=none] () at (\x+1.2,\y) {$=$};

\def\x{2.5}
\def\y{0}
\node[draw=none] at (\x,\y+0.6) {$\P\pinv$};
\node[tensor,fill=green!60!red!40!white] (P) at (\x,\y) {$\dagger$};
\draw[edge] (P) -- (\x+0.75,\y);
\draw[edge] (P) -- (\x-0.75,\y);
\node[draw=none] () at (\x-0.5,\y+0.2) {\textcolor{gray}{$n$}};
\node[draw=none] () at (\x+0.8,\y+0.2) {\textcolor{gray}{$\Pref$}};

\def\x{4}
\def\y{0}

\node[draw=none] at (\x,\y+0.6) {$\Ht^\sigma_{\Pref,\Suff}$};
\node[tensor,fill=blue!60!green!40!white] (G0) at (\x,\y) {};
\draw[edge] (G0) -- (\x+0.75,\y);
\draw[edge] (G0) -- (\x-0.75,\y);
\node[draw=none] () at (\x+0.8,\y+0.2) {\textcolor{gray}{$p$}};

\def\x{5.5}
\def\y{0}
\node[draw=none] at (\x,\y+0.6) {$(\tenmat{\St}{1})\pinv$};
\node[tensor,fill=blue!60!red!60!white] (S) at (\x,\y) {$\dagger$};
\draw[edge] (S) -- (\x+0.75,\y);
\draw[edge] (S) -- (\x-0.75,\y);
\path  (S)  edge [bend left,edge,in=90,out=90,distance=0.5cm]  (G0);
\node[draw=none] () at (\x+0.5,\y+0.2) {\textcolor{gray}{$n$}};
\node[draw=none] () at (\x-0.75,\y-0.8) {\textcolor{gray}{$\Suff$}};

\end{tikzpicture} 

\end{center}

\subsection{Recurrent Neural Networks}

\emph{Recurrent neural networks}~(RNN) are a class of neural networks designed to handle sequential data. A RNN takes as input a sequence (of arbitrary length) of elements from an input space $\Xcal$ and outputs an element in the output space $\Ycal$. Thus a RNN computes a function from $\Xcal^*$, the set of all finite-length sequences of elements of $\Xcal$, to $\Ycal$. In most applications, $\Xcal$ is a vector space, typically $\R^d$. When the input of the problem are sequences of symbols from a finite alphabet $\Sigma$, so-called \emph{one-hot} encoding are often used to embed $\Sigma$ into $\R^{|\Sigma|}$ by representing each symbol in $\Sigma$ by one of the canonical basis vector. 

There are several ways to describe recurrent neural networks. We opt here for a relatively abstract one, which will allow us to seamlessly draw connections with the vv-WFA model presented in the previous section. We first introduce the general notion of a recurrent model, which encompasses many of the models used in machine learning to handle sequential data~(most RNN architectures, hidden Markov models, WFA, etc.).

\begin{definition}
Let $\Xcal$ and $\Ycal$ be the input and output space, respectively. A \emph{recurrent model} with $n$ states is given by a tuple $R=(\frec,\fout,\h_0)$ where $\frec:\Xcal\times\R^n \to \R^n$ is the \emph{recurrent function}, $\fout:\R^n\to\Ycal$ is the \emph{output function} and $\h_0\in\R^n$ is the \emph{initial state}. A recurrent model $R$ computes a function $f_R:\Xcal^*\to\Ycal$ defined by the (recurrent) relation:
$$
f_R(x_1x_2\cdots x_k) = \fout(\h_k)\ \ \ \text{where }  \h_t = \frec(x_t,\h_{t-1})\ \text{for } 1\leq t\leq k
$$
for all $k\geq 0$ and $x_1,x_2,\dots,x_k\in\Xcal$.
\end{definition}

One can easily check that this definition encompasses vv-WFA. Indeed, a WFA $A=\vvwa$ is a recurrent model with $n$ states where $\Xcal=\Sigma$, $\Ycal=\R^p$, $\h_0=\szero$ and the recurrent and output functions are given by
$$
\frec(\sigma,\h) = (\A^\sigma)^\top\h\ \ \ \text{and}\ \ \ \fout(\h)= \vvsinf\h
$$
for all $x\in\Sigma,\ \h\in\R^n$.

Many architectures of recurrent neural networks have been proposed and used in practice. In this paper, we focus on vanilla RNN, also known as Elman network~\citep{elman1990finding}, and second-order RNN~(2-RNN)~\citep{giles1990higher,pollack1991induction,lee1986machine}\footnote{Second-order recurrent architectures have been successfully used more recently, see \eg \citep{sutskever2011generating} and \citep{wu2016multiplicative}.}, which can be seen as a multilinear extension of vanilla RNN. We now give the formal definitions of these two models.

\begin{definition}
A \emph{first-order RNN}~(or vanilla RNN) with $n$ states~(or, equivalently, $n$ hidden neurons) is a recurrent model $R=(\frec,\fout,\h_0)$ with input space $\Xcal=\R^d$ and output space $\Ycal=\R^p$. It computes a function $f_R:(\R^d)^*\to \R^p$ defined by $f_R(\x_1,\dots, \x_k)=\fout(\h_k)$, where the recurrent and output functions are defined by
$$\h_t=\frec(\x_t,\h_{t-1}) = z_{rec}(\U\x_t + \V\h_{t-1})\ \ \ \text{ and }\ \ \ \y_t=\fout(\h_t)=z_{out}(\Wb\h_{t}).$$
The parameters of a first-order RNN are:
\begin{itemize}
    \item the initial state $\h_0\in\R^n$, 
    \item the weight matrices $\U\in\R^{n\times d}$,  $\V\in\R^{n\times n}$ and $\Wb\in\R^{p\times n}$,
    \item the activation functions $z_{rec}:\R^n\to\R^n$ and $z_{out}:\R^p\to\R^p$.
\end{itemize}
\end{definition}

For the sake of simplicity, we omitted the bias vectors usually included in the definition of first-order RNN. Note however that this is without loss of generality when $z_{rec}$ is either a rectified linear unit or the identity~(which will be the cases considered in this paper). Indeed, for any recurrent model with $n$ states $R=(\frec,\fout,\h_0)$ with input space $\Xcal=\R^d$ and output space $\Ycal=\R^p$ defined by
$$\h_t=\frec(\x_t,\h_{t-1}) = z_{rec}(\U\x_t + \V\h_{t-1} + \bb)\ \ \ \text{ and }\ \ \ \fout(\h_t)=z_{out}(\Wb\h_{t}+\cb)$$
one can append a $1$ to all input vectors, $\tilde{\x}_t= (\x_t\ 1)^\top$, and define a new recurrent model with $n+1$ states $\tilde{R}=(\tilde{\frec},\tilde{\fout},\tilde{\h}_0)$ with input space $\Xcal=\R^{d+1}$ and output space $\Ycal=\R^p$ defined by
$$
\tilde{\h}_t=\frec(\tilde{\x}_t,\tilde{\h}_{t-1}) = z_{rec}(\tilde{\U}\tilde{\x}_t + \tilde{\V}\tilde{\h}_{t-1}),\ \ \fout(\tilde{\h}_t)=z_{out}(\tilde{\Wb}\tilde{\h}_{t})\ \ \text{and } \tilde{\h}_0=(\h_0\ 1)^\top
$$
computing the same function.

\begin{definition}
A \emph{second-order RNN}~(2-RNN) with $n$ states is a recurrent model $R=(\frec,\fout,\h_0)$ with input space $\Xcal=\R^d$ and output space $\Ycal=\R^p$. It computes a function $f_R:(\R^d)^*\to \R^p$ defined by $f_R(\x_1,\dots, \x_k)=\fout(\h_k)$, where the recurrent and output functions are defined by
$$\h_t=\frec(\x_t,\h_{t-1}) = z_{rec}(\At\ttv{1}\h_{t-1}\ttv{2}\x_t)\ \ \ \text{ and }\ \ \ \y_t=\fout(\h_t)=z_{out}(\Wb\h_{t}).$$
The parameters of a second-order RNN are:
\begin{itemize}
    \item the initial state $\h_0\in\R^n$, 
    \item the weight tensor $\At\in\R^{n\times d\times n}$ and output matrix $\Wb\in\R^{p\times n}$,
    \item the activation functions $z_{rec}:\R^n\to\R^n$ and $z_{out}:\R^p\to\R^p$.
\end{itemize}
A linear 2-RNN $R$ with $n$ states is called \emph{minimal} if its number of states is minimal~(\ie any linear 2-RNN computing $f_R$ has
at least $n$ states).
\end{definition}

In the remaining of the paper, we will define a second-order RNN using its parameters, i.e. $R=(\h_0,\At,\Wb,z_{rec},z_{out})$. In the particular case where the activation functions are linear~(i.e. equal to the identity function), we will omit them from the definition, e.g. $R=(\h_0,\At,\Wb)$ defines a linear second-order RNN.

The recurrent activation function $z_{rec}$ of a RNN is usually a componentwise non-linear function such as a hyperbolic tangent or rectified linear unit, while the output activation function often depends on the task~(the softmax function being the most popular for classification and language modeling tasks).

One can see that the difference between first-order and second-order RNN only lies in the recurrent function. For first-order RNN, the pre-activation $\ab_t=\U\x_t + \V\h_{t-1} + \bb$ is a linear function of $\x_t$ and $\h_{t-1}$, while for second-order RNN the pre-activation $\ab_t=\At\ttv{1}\h_{t-1}\ttv{2}\x_t$ is a bilinear map applied to $\x_t$ and $\h_{t-1}$~(hence the \emph{second-order} denomination). 

It is worth mentioning that second-order RNN are often defined with additional parameters to account for first-order interactions and bias terms:
$$\h_t=\frec(\x_t,\h_{t-1}) = z_{rec}(\At\ttv{1}\h_{t-1}\ttv{2}\x_t + \U\x_t + \V\h_{t-1} + \bb).$$
The definition we use here is conceptually simpler and without loss of generality~(similarly to the omission of the bias vectors in the definition of first-order RNN). Indeed, when $z_{rec}$ is either the identity or a rectified linear unit, one can always append a 1 to all input vectors and augment the state space by one state to obtain a 2-RNN computing the same function. It follows from this discussion that 2-RNN are a strict generalization of vanilla RNN: any function that can be computed by a vanilla RNN can be computed by a 2-RNN~(provided that all input vectors are appended a constant entry equal to one).

The recurrent and output functions of a 2-RNN can be represented by the following simple tensor networks:
\begin{center}
\begin{tikzpicture}

\tikzset{tensor/.style = {minimum size = 0.7cm,shape = circle,thick,draw=black,fill=blue!60!green!40!white,inner sep = 0pt}, edge/.style   = {thick,line width=.4mm},every loop/.style={}}
\scalebox{0.9}{
\def\x{0}
\def\y{0}
\node[tensor] (G0) at (\x,\y) {$\h_t$};
\draw[edge] (G0) -- (\x+1,\y);

\node[draw=none] () at (\x+1.8,\y) {$=\ $\scalebox{1.1}{$z_{rec}$}$\left(\rule{0pt}{25pt}\right. $};

\node[draw=none] () at (\x+0.8,\y+0.2) {\textcolor{gray}{$n$}};

\def\x{3}
\def\y{0}
\node[tensor] (ht) at (\x,\y) {$\h_{t-1}$};
\node[tensor,fill=lightgray] (A) at (\x+1.5,\y) {$\At$};
\node[tensor,fill=green!60!red!40!white] (x) at (\x+1.5,\y-1.25) {$\xb_{t}$};
\draw[edge] (ht) -- (A);
\draw[edge] (x) -- (A);
\draw[edge] (A) -- (\x+2.5,\y);

\node[draw=none] () at (\x+0.8,\y+0.2) {\textcolor{gray}{$n$}};
\node[draw=none] () at (\x+2.2,\y+0.2) {\textcolor{gray}{$n$}};
\node[draw=none] () at (\x+1.8,\y-0.6) {\textcolor{gray}{$d$}};
\node[draw=none] () at (\x+2.7,\y) {$\left.\rule{0pt}{25pt}\right) $};

\def\x{7}
\def\y{0}
\node[tensor,fill=green!60!red!40!white] (G1) at (\x,\y) {$\yb_t$};
\draw[edge] (G1) -- (\x+1,\y);

\node[draw=none] () at (\x+1.8,\y) {$=\ $\scalebox{1.1}{$z_{out}$}$\left(\rule{0pt}{25pt}\right. $};

\node[draw=none] () at (\x+0.8,\y+0.2) {\textcolor{gray}{$p$}};

\def\x{10}
\def\y{0}
\node[tensor] (htt) at (\x,\y) {$\h_{t}$};
\node[tensor,fill=lightgray] (W) at (\x+1.5,\y) {$\Wb$};
\draw[edge] (htt) -- (W);
\draw[edge] (W) -- (\x+2.5,\y);

\node[draw=none] () at (\x+0.8,\y+0.2) {\textcolor{gray}{$n$}};
\node[draw=none] () at (\x+2.2,\y+0.2) {\textcolor{gray}{$p$}};
\node[draw=none] () at (\x+2.7,\y) {$\left.\rule{0pt}{25pt}\right) $};
}

\end{tikzpicture}
\end{center}

By introducing the notion of recurrent models, we presented a unified view of WFA and first and second-order RNN. All these sequential models are recurrent models and differ in the way their recurrent and output functions are defined. The difference between WFA and RNN can thus be summarized by the fact that the recurrent and output functions of a WFA are linear, whereas they are non-linear maps for RNN. In essence, one could say that RNN are non-linear extensions of WFA. In Section~\ref{sec:WFA.2RNN.equivalence}, we will formalize this intuition by proving the exact equivalence between the classes of functions that can be computed by WFA and second-order RNN with linear activation functions.

\section{Weighted Automata and Tensor Networks}
In this section, we present connections between weighted automata and tensor networks. In particular, we will show that the computation of a WFA on a sequence is intrinsically connected to the matrix product states model used in quantum physics and the tensor train decomposition. This connection will allow us to unravel a fundamental structure in the Hankel matrix of a function computed by a WFA: in addition to being  low rank, we will show that the Hankel matrix can be decomposed into sub-blocks which are all matricizations of tensors with low tensor train rank. We will then leverage this structure to design an efficient spectral learning algorithm for WFA relying on efficient computations of pseudo-inverse of matrices given in the tensor train format.

\subsection{Tensor Train Structure of the Hankel Matrix}
\label{sec:TT.struct.hankel}
For the sake of simplicity, we will consider scalar valued WFA in this section but all the results we present can be straightforwardly extended to vv-WFA. Let $A=\wa$ be a WFA with $n$ states. Recall that $A$ computes a function $f_A: \Sigma^* \to \R$ defined by
$$f_A(x_1x_2\cdots x_k) = \szero^\top \A^{x_1}\A^{x_2}\cdots\A^{x_k}\sinf$$
for any $k\geq 0$ and $x_1,x_2,\cdots, x_k\in\Sigma$. The computation of a WFA on a sequence can be represented by the following tensor network:

$$
f(x_1x_2\cdots x_k)=
\begin{tikzpicture}[baseline=-0.5ex]
\tikzset{tensor/.style = {minimum size = 0.4cm,shape = circle,thick,draw=black,fill=blue!60!green!40!white,inner sep = 0pt}, edge/.style   = {thick,line width=.4mm},every loop/.style={}}

\node[tensor] (G0) at (0,0) {};
\node[draw=none,right=0.2cm of G0] (G0right) {};
\draw[edge] (G0) -- (G0right);
\node[draw=none] at (0,0.5) {$\szero$};

\draw[edge] (0.39999999999999997,0) -- (1.5999999999999999,0);
\node[tensor] (G1) at (1.0,0) {};
\node[draw=none,left=0.2cm of G1] (G1left) {};
\draw[edge] (G1) -- (G1left);
\node[draw=none,right=0.2cm of G1] (G1right) {};
\draw[edge] (G1) -- (G1right);
\node[draw=none] at (1.0,0.5) {$\A^{x_1}$};

\node[tensor] (G2) at (2.0,0) {};
\node[draw=none,left=0.2cm of G2] (G2left) {};
\draw[edge] (G2) -- (G2left);
\node[draw=none,right=0.2cm of G2] (G2right) {};
\draw[edge] (G2) -- (G2right);
\node[draw=none] at (2.0,0.5) {$\A^{x_2}$};

\node[draw=none] (G3) at (2.8,0) {$\cdots$};
\node[draw=none] at (2.8,0.5) {$\cdots$};

\node[tensor] (G4) at (3.5999999999999996,0) {};
\node[draw=none,left=0.2cm of G4] (G4left) {};
\draw[edge] (G4) -- (G4left);
\node[draw=none,right=0.2cm of G4] (G4right) {};
\draw[edge] (G4) -- (G4right);
\node[draw=none] at (3.5999999999999996,0.5) {$\A^{x_{k-1}}$};

\draw[edge] (3.999999999999999,0) -- (5.2,0);
\node[tensor] (G5) at (4.6,0) {};
\node[draw=none,left=0.2cm of G5] (G5left) {};
\draw[edge] (G5) -- (G5left);
\node[draw=none,right=0.2cm of G5] (G5right) {};
\draw[edge] (G5) -- (G5right);
\node[draw=none] at (4.6,0.5) {$\A^{x_k}$};

\node[tensor] (G6) at (5.6,0) {};
\node[draw=none,left=0.2cm of G6] (G6left) {};
\draw[edge] (G6) -- (G6left);
\node[draw=none] at (5.6,0.5) {$\sinf$};

\end{tikzpicture}
$$

By stacking the transition matrices $\{\A^\sigma\}_{\sigma\in\Sigma}$ into a third order tensor $\At\in\R^{n\times\Sigma\times n}$ defined by 
$$\At_{:,\sigma,:}=\A^{\sigma}\ \ \ \text{ for all }\sigma\in\Sigma,$$
this computation can be rewritten into 

$$f(x_1x_2\cdots x_k)=
\begin{tikzpicture}[baseline=-0.5ex]
\tikzset{tensor/.style = {minimum size = 0.4cm,shape = circle,thick,draw=black,fill=blue!60!green!40!white,inner sep = 0pt}, edge/.style   = {thick,line width=.4mm},every loop/.style={}}

\tikzset{tensor/.style = {minimum size = 0.4cm,shape = circle,thick,draw=black,fill=blue!60!green!40!white,inner sep = 0pt}}
\node[tensor,fill=blue!60!green!40!white] (G0) at (0,0) {};
\node[draw=none,right=0.2cm of G0] (G0right) {};
\draw[edge] (G0) -- (G0right);
\node[draw=none] at (0,0.5) {$\szero$};

\node[tensor,fill=blue!60!green!40!white] (G1) at (0.8,0) {};
\draw[edge] (G1) -- (0.8,-0.5);
\node[draw=none,left=0.2cm of G1] (G1left) {};
\draw[edge] (G1) -- (G1left);
\node[draw=none,right=0.2cm of G1] (G1right) {};
\draw[edge] (G1) -- (G1right);
\node[draw=none] at (0.8,0.5) {$\At$};
\node[draw=none] at (0.8,-0.7) {$x_1$};

\node[tensor,fill=blue!60!green!40!white] (G2) at (1.6,0) {};
\draw[edge] (G2) -- (1.6,-0.5);
\node[draw=none,left=0.2cm of G2] (G2left) {};
\draw[edge] (G2) -- (G2left);
\node[draw=none,right=0.2cm of G2] (G2right) {};
\draw[edge] (G2) -- (G2right);
\node[draw=none] at (1.6,0.5) {$\At$};
\node[draw=none] at (1.6,-0.7) {$x_2$};

\node[draw=none] (G3) at (2.4000000000000004,0) {$\cdots$};
\node[draw=none] at (2.4000000000000004,0.5) {$\cdots$};

\node[tensor,fill=blue!60!green!40!white] (G4) at (3.2,0) {};
\draw[edge] (G4) -- (3.2,-0.5);
\node[draw=none,left=0.2cm of G4] (G4left) {};
\draw[edge] (G4) -- (G4left);
\node[draw=none,right=0.2cm of G4] (G4right) {};
\draw[edge] (G4) -- (G4right);
\node[draw=none] at (3.2,0.5) {$\At$};
\node[draw=none] at (3.2,-0.7) {$x_{k-1}$};

\node[tensor,fill=blue!60!green!40!white] (G5) at (4.0,0) {};
\draw[edge] (G5) -- (4.0,-0.5);
\node[draw=none,left=0.2cm of G5] (G5left) {};
\draw[edge] (G5) -- (G5left);
\node[draw=none,right=0.2cm of G5] (G5right) {};
\draw[edge] (G5) -- (G5right);
\node[draw=none] at (4.0,0.5) {$\At$};
\node[draw=none] at (4.0,-0.7) {$x_k$};

\node[tensor,fill=blue!60!green!40!white] (G6) at (4.8,0) {};
\node[draw=none,left=0.2cm of G6] (G6left) {};
\draw[edge] (G6) -- (G6left);
\node[draw=none] at (4.8,0.5) {$\sinf$};
\end{tikzpicture}
$$

This graphically shows the tight connection between WFA and the tensor train decomposition. More formally, for any integer $l$, let us define the $l$th order Hankel tensor $\Ht^{(l)}\in\R^{\Sigma \times \Sigma \times\cdots \Sigma}$ by  
\begin{equation}\label{eq:def.hankel.tensor}
    \Ht^{(l)}_{\sigma_1,\sigma_2,\cdots,\sigma_l} = f(\sigma_1\sigma_2\cdots\sigma_l)\text{ for all }\sigma_1,\cdots\sigma_l\in\Sigma.
\end{equation} 
Then, one can easily check that each such Hankel tensor admits the following rank $n$ tensor train decomposition:
\begin{align*}
\Ht^{(l)} 
&= 
\begin{tikzpicture}[baseline=-0.5ex]

\tikzset{tensor/.style = {minimum size = 0.4cm,shape = circle,thick,draw=black,fill=blue!60!green!40!white,inner sep = 0pt}, edge/.style   = {thick,line width=.4mm},every loop/.style={}}
\node[tensor,fill=blue!60!green!40!white] (G0) at (0,0) {};
\node[draw=none,right=0.2cm of G0] (G0right) {};
\draw[edge] (G0) -- (G0right);
\node[draw=none] at (0,0.5) {$\szero$};

\node[tensor,fill=blue!60!green!40!white] (G1) at (0.8,0) {};
\draw[edge] (G1) -- (0.8,-0.5);
\node[draw=none,left=0.2cm of G1] (G1left) {};
\draw[edge] (G1) -- (G1left);
\node[draw=none,right=0.2cm of G1] (G1right) {};
\draw[edge] (G1) -- (G1right);
\node[draw=none] at (0.8,0.5) {$\At$};

\node[tensor,fill=blue!60!green!40!white] (G2) at (1.6,0) {};
\draw[edge] (G2) -- (1.6,-0.5);
\node[draw=none,left=0.2cm of G2] (G2left) {};
\draw[edge] (G2) -- (G2left);
\node[draw=none,right=0.2cm of G2] (G2right) {};
\draw[edge] (G2) -- (G2right);
\node[draw=none] at (1.6,0.5) {$\At$};

\node[draw=none] (G3) at (2.4000000000000004,0) {$\cdots$};
\node[draw=none] at (2.4000000000000004,0.5) {$\cdots$};

\node[tensor,fill=blue!60!green!40!white] (G4) at (3.2,0) {};
\draw[edge] (G4) -- (3.2,-0.5);
\node[draw=none,left=0.2cm of G4] (G4left) {};
\draw[edge] (G4) -- (G4left);
\node[draw=none,right=0.2cm of G4] (G4right) {};
\draw[edge] (G4) -- (G4right);
\node[draw=none] at (3.2,0.5) {$\At$};

\node[tensor,fill=blue!60!green!40!white] (G5) at (4.0,0) {};
\draw[edge] (G5) -- (4.0,-0.5);
\node[draw=none,left=0.2cm of G5] (G5left) {};
\draw[edge] (G5) -- (G5left);
\node[draw=none,right=0.2cm of G5] (G5right) {};
\draw[edge] (G5) -- (G5right);
\node[draw=none] at (4.0,0.5) {$\At$};

\node[tensor,fill=blue!60!green!40!white] (G6) at (4.8,0) {};
\node[draw=none,left=0.2cm of G6] (G6left) {};
\draw[edge] (G6) -- (G6left);
\node[draw=none] at (4.8,0.5) {$\sinf$};

\draw[decoration={{brace}},decorate,thick] (0.6,0.7) -- node[above=3pt] {{\tiny $l$ times}} (4.3,0.7);

\end{tikzpicture}\\
&=
\begin{tikzpicture}[baseline=-0.5ex]

\tikzset{tensor/.style = {minimum size = 0.4cm,shape = circle,thick,draw=black,fill=blue!60!green!40!white,inner sep = 0pt}, edge/.style   = {thick,line width=.4mm},every loop/.style={}}
\node[tensor,fill=blue!60!green!40!white] (G0) at (0,0) {};
\draw[edge] (G0) -- (0,-0.5);
\node[draw=none,right=0.2cm of G0] (G0right) {};
\draw[edge] (G0) -- (G0right);
\node[draw=none] at (0,0.5) {$\At\ttv{1}\szero$};

\draw[edge] (0.4,0) -- (2.2,0);
\node[tensor,fill=blue!60!green!40!white] (G1) at (1.3,0) {};
\draw[edge] (G1) -- (1.3,-0.5);
\node[draw=none,left=0.2cm of G1] (G1left) {};
\draw[edge] (G1) -- (G1left);
\node[draw=none,right=0.2cm of G1] (G1right) {};
\draw[edge] (G1) -- (G1right);
\node[draw=none] at (1.3,0.5) {$\At$};

\node[draw=none] (G2) at (2.6,0) {$\cdots$};
\node[draw=none] at (2.6,0.5) {$\cdots$};

\draw[edge] (3.0000000000000004,0) -- (4.800000000000001,0);
\node[tensor,fill=blue!60!green!40!white] (G3) at (3.9000000000000004,0) {};
\draw[edge] (G3) -- (3.9000000000000004,-0.5);
\node[draw=none,left=0.2cm of G3] (G3left) {};
\draw[edge] (G3) -- (G3left);
\node[draw=none,right=0.2cm of G3] (G3right) {};
\draw[edge] (G3) -- (G3right);
\node[draw=none] at (3.9000000000000004,0.5) {$\At$};

\node[tensor,fill=blue!60!green!40!white] (G4) at (5.2,0) {};
\draw[edge] (G4) -- (5.2,-0.5);
\node[draw=none,left=0.2cm of G4] (G4left) {};
\draw[edge] (G4) -- (G4left);
\node[draw=none] at (5.2,0.5) {$\At\ttv{3}\sinf$};

\end{tikzpicture} \\
&=
\TT{\At\ttv{1}\szero,\At,\cdots,\At,\At\ttv{3}\sinf}
\end{align*}

It follows that the Hankel matrix of a recognizable function can be decomposed into sub-blocks which are all matricization of Hankel tensors with low tensor train rank. To the best of our knowledge, this is a novel result that has not been noticed in the past. We conclude this section by formalizing this result in the following theorem.

\begin{theorem}
Let $f:\Sigma^*\to\R$ be a function computed by a WFA with $n$ states and let $\H\in\R^{\Sigma^*\times\Sigma^*}$ be its Hankel matrix defined by $\H_{u,v}=f(uv)$ for all $u,v\in\Sigma^*$. Furthermore, for any integer $l$, let $\Ht^{(l)}\in\R^{\Sigma \times \Sigma \times\cdots \times\Sigma}$ be the $l$th order tensor defined by  $\Ht^{(l)}_{\sigma_1,\sigma_2,\cdots,\sigma_l} = f(\sigma_1\sigma_2\cdots\sigma_l)$.

Then, the Hankel matrix $\H$ can be decomposed into sub-blocks, each sub-block being the matricization of a tensor of tensor train rank at most $n$. More precisely, each of these sub-blocks is equal to $\tenmatgen{\Ht^{(l)}}{k,l-k}$ for some values of $l$ and $k$, and 
 each Hankel tensor $\Ht^{(l)}$ has tensor train rank at most $n$.
\end{theorem}
\begin{proof}
For each $m,k\in\Nbb$, let $\H^{(m,k)}\in\Rbb^{\Sigma^m\times \Sigma^k}$ denote the sub-block of the Hankel matrix with prefixes $\Sigma^m$ and suffixes $\Sigma^k$. It is easy to check that the Hankel matrix $\H\in\R^{\Sigma^*\times\Sigma^*}$ can be partitioned into the sub-blocks $\H^{(m,k)}$ for $m,k\in\Nbb$:
$$ \H
=
\bmat{
\H^{(0,0)} & \H^{(0,1)} & \H^{(0,2)} & \H^{(0,3)} & \cdots \\
\H^{(1,0)} & \H^{(1,1)} & \H^{(1,2)} & \H^{(1,3)} & \cdots \\
\H^{(2,0)} & \H^{(2,1)} & \H^{(2,2)} & \H^{(2,3)} & \cdots \\
\vdots & \vdots & \vdots & \vdots & \ddots 
}\ .
$$
Now, by definition of the tensors $\Hten^{(l)}$, we have $\H^{(m,k)} = \tenmatgen{\Ht^{(m+k)}}{m,k}$. Moreover, let $A=\wa$ be a WFA with $n$ states computing $f$ and let $\At\in\R^{n\times \Sigma \times n}$ be the 3rd order tensor defined by $\At_{:,\sigma,:}=\A^{\sigma}$ for each $\sigma\in\Sigma$. For any $m,k\in\Nbb$ and any $\sigma_{1},\cdots,\sigma_{m+k}\in\Sigma$, we have  
\begin{align*}
\H^{(m,k)}_{\sigma_{1},\cdots,\sigma_{m+k}}  
&= (\tenmatgen{\Ht^{(m+k)}}{m,k})_{\sigma_{1}\cdots\sigma_m,\sigma_{m+1}\cdots\sigma_{m+k}} \\
&= f(\sigma_1,\sigma_2,\cdots,\sigma_{m+k}) \\
&= \szero^\top \A^{\sigma_1}\A^{\sigma_2}\cdots\A^{\sigma_{m+k}}\sinf \\
&= \szero^\top \At_{:,\sigma_1,:}\At_{:,\sigma_2,:}\cdots\At_{:,\sigma_{m+k},:}\sinf\\
&= \TT{\At\ttv{1}\szero,\At,\cdots,\At,\At\ttv{3}\sinf}_{{\sigma_{1},\cdots,\sigma_{m+k}}}.
\end{align*}
It follows that 
$$
\H^{(m,k)} = \tenmatgen{\Ht^{(m+k)}}{m,k} = \tenmatgen{{\TT{\At\ttv{1}\szero,\overbrace{\At,\cdots,\At,}^{m+k-2\text{ times}}\At\ttv{3}\sinf}}}{m,k}
$$
and thus that each sub-block $\H^{(m,k)}$ is a matricization of a tensor of tensor train rank at most $n$.
\end{proof}

\subsection{Spectral Learning in the Tensor Train Format}
\label{sec:spectral.learning.TT}

We now present how the tensor train structure of the Hankel matrix can be leveraged to significantly improve the computational complexity of steps 3 and 4 of the spectral learning algorithm described in Section~\ref{sec:wfa.and.speclearn}. 
These two steps consist in first computing a low rank approximation of the Hankel sub-block
$$\H_{\Pref,\Suff}\simeq \P\S$$
before estimating the parameters of the WFA using simple pseudo-inverse and matrix product computations
$$\szero^\top = \hb_\Suff^\top \S\pinv,\  \sinf =\P\inv\hb_\Pref\ \text{and}\ 
   \A^\sigma =  \P\pinv\H^\sigma_{\Pref,\Suff}\S\pinv  \ \text{ for each }\sigma\in\Sigma$$
where the Hankel sub-blocks are defined by
$$(\hb_\Pref)_u=f(u),\ \ (\hb_\Suff)_v=f(v),\ \ (\H_{\Pref,\Suff})_{u,v}=f(uv)\ \text{ and }\ (\H^\sigma_{\Pref,\Suff})_{u,v}=f(u\sigma v)$$
for all $\sigma\in\Sigma,u\in\Pref,v\in\Suff$.
Note that we again focus on scalar-valued WFA for the sake of clarity~(i.e. $A=\wa$) but the results we present can be straightforwardly extended to vector-valued WFA. Using tensor networks, these two steps are described as follows:
\begin{center}
    \begin{center}
    
\begin{tikzpicture}

\tikzset{tensor/.style = {minimum size = 0.4cm,shape = circle,thick,draw=black,fill=blue!60!green!40!white,inner sep = 0pt}, edge/.style   = {thick,line width=.4mm},every loop/.style={}}

\def\x{0}
\def\y{0}
\node[draw=none] at (\x,\y+0.7) {$\H_{\Pref,\Suff}$};
\node[tensor] (G0) at (\x,\y) {};
\draw[edge] (G0) -- (\x+0.75,\y);
\draw[edge] (G0) -- (\x-0.75,\y);
\node[draw=none] () at (\x-0.5,\y+0.2) {\textcolor{gray}{$\Pref$}};
\node[draw=none] () at (\x+0.5,\y+0.2) {\textcolor{gray}{$\Suff$}};

\node[draw=none] () at (\x+1.2,\y) {$\simeq$};

\def\x{2.5}
\def\y{0}
\node[draw=none] at (\x,\y+0.7) {$\P$};
\node[tensor,fill=green!60!red!40!white] (P) at (\x,\y) {};
\draw[edge] (P) -- (\x+0.75,\y);
\draw[edge] (P) -- (\x-0.75,\y);
\node[draw=none] () at (\x-0.5,\y+0.2) {\textcolor{gray}{$\Pref$}};
\node[draw=none] () at (\x+0.5,\y+0.2) {\textcolor{gray}{$n$}};

\def\x{3.5}
\def\y{0}
\node[draw=none] at (\x,\y+0.7) {$\S$};
\node[tensor,fill=blue!60!red!60!white] (S) at (\x,\y) {};
\draw[edge] (S) -- (\x+0.75,\y);
\draw[edge] (S) -- (\x-0.75,\y);
\node[draw=none] () at (\x+0.5,\y+0.2) {\textcolor{gray}{$\Suff$}};

\end{tikzpicture} 









 \medskip

\begin{tikzpicture}

\tikzset{tensor/.style = {minimum size = 0.4cm,shape = circle,thick,draw=black,fill=blue!60!green!40!white,inner sep = 0pt}, edge/.style   = {thick,line width=.4mm},every loop/.style={}}

\def\x{0}
\def\y{0}
\node[draw=none] at (\x,\y+0.5) {$\szero$};
\node[tensor,fill=lightgray] (G0) at (\x,\y) {};
\draw[edge] (G0) -- (\x+0.75,\y);
\node[draw=none] () at (\x+0.5,\y+0.2) {\textcolor{gray}{$n$}};

\node[draw=none] () at (\x+1.2,\y) {$=$};

\def\x{2}
\def\y{0}

\node[draw=none] at (\x,\y+0.6) {$\hb_{\Suff}$};
\node[tensor,fill=blue!60!green!40!white] (G0) at (\x,\y) {};
\draw[edge] (G0) -- (\x+0.75,\y);
\node[draw=none] () at (\x+0.8,\y+0.2) {\textcolor{gray}{$S$}};

\def\x{3.5}
\def\y{0}
\node[draw=none] at (\x,\y+0.6) {$\S\pinv$};
\node[tensor,fill=blue!60!red!60!white] (S) at (\x,\y) {$\dagger$};
\draw[edge] (S) -- (\x+0.75,\y);
\draw[edge] (S) -- (\x-0.75,\y);
\node[draw=none] () at (\x+0.5,\y+0.2) {\textcolor{gray}{$n$}};

\def\x{7}
\def\y{0}
\node[draw=none] at (\x,\y+0.5) {$\sinf$};
\node[tensor,fill=lightgray] (G1) at (\x,\y) {};
\draw[edge] (G1) -- (\x-0.75,\y);
\node[draw=none] () at (\x-0.5,\y+0.2) {\textcolor{gray}{$n$}};

\node[draw=none] () at (\x+0.7,\y) {$=$};

\def\x{8.8}
\def\y{0}
\node[draw=none] at (\x,\y+0.6) {$\P\pinv$};
\node[tensor,fill=green!60!red!40!white] (P) at (\x,\y) {$\dagger$};
\draw[edge] (P) -- (\x+0.75,\y);
\draw[edge] (P) -- (\x-0.75,\y);
\node[draw=none] () at (\x-0.5,\y+0.2) {\textcolor{gray}{$n$}};
\node[draw=none] () at (\x+0.8,\y+0.2) {\textcolor{gray}{$\Pref$}};

\def\x{10.3}
\def\y{0}

\node[draw=none] at (\x,\y+0.6) {$\hb_{\Pref}$};
\node[tensor,fill=blue!60!green!40!white] (G0) at (\x,\y) {};
\draw[edge] (G0) -- (\x-0.75,\y);

\end{tikzpicture} 

\medskip

\begin{tikzpicture}

\tikzset{tensor/.style = {minimum size = 0.4cm,shape = circle,thick,draw=black,fill=blue!60!green!40!white,inner sep = 0pt}, edge/.style   = {thick,line width=.4mm},every loop/.style={}}

\def\x{0}
\def\y{0}
\node[draw=none] at (\x,\y+0.5) {$\A^\sigma$};
\node[tensor,fill=lightgray] (G1) at (\x,\y) {};
\draw[edge] (G1) -- (\x+0.75,\y);
\draw[edge] (G1) -- (\x-0.75,\y);
\node[draw=none] () at (\x+0.5,\y+0.2) {\textcolor{gray}{$n$}};
\node[draw=none] () at (\x-0.5,\y+0.2) {\textcolor{gray}{$n$}};

\node[draw=none] () at (\x+1.2,\y) {$=$};

\def\x{2.5}
\def\y{0}
\node[draw=none] at (\x,\y+0.6) {$\P\pinv$};
\node[tensor,fill=green!60!red!40!white] (P) at (\x,\y) {$\dagger$};
\draw[edge] (P) -- (\x+0.75,\y);
\draw[edge] (P) -- (\x-0.75,\y);
\node[draw=none] () at (\x-0.5,\y+0.2) {\textcolor{gray}{$n$}};
\node[draw=none] () at (\x+0.8,\y+0.2) {\textcolor{gray}{$\Pref$}};

\def\x{4}
\def\y{0}

\node[draw=none] at (\x,\y+0.6) {$\H^\sigma_{\Pref,\Suff}$};
\node[tensor,fill=blue!60!green!40!white] (G0) at (\x,\y) {};
\draw[edge] (G0) -- (\x+0.75,\y);
\draw[edge] (G0) -- (\x-0.75,\y);
\node[draw=none] () at (\x+0.8,\y+0.2) {\textcolor{gray}{$\Suff$}};

\def\x{5.5}
\def\y{0}
\node[draw=none] at (\x,\y+0.6) {$\S\pinv$};
\node[tensor,fill=blue!60!red!60!white] (S) at (\x,\y) {$\dagger$};
\draw[edge] (S) -- (\x+0.75,\y);
\draw[edge] (S) -- (\x-0.75,\y);
\node[draw=none] () at (\x+0.5,\y+0.2) {\textcolor{gray}{$n$}};

\end{tikzpicture} 

\end{center}
\end{center}

We now focus on the case where the basis of prefixes and suffixes are both equal to the set of all sequences of length $l$  for some integer $l$, i.e. $\Pref=\Suff=\Sigma^l$. 
A first important observation is that, in this case, the Hankel sub-block $\H_{\Pref,\Suff}$ is a matricization of the $2l$-th order Hankel tensor $\Ht^{(2l)}\in\R^{\Sigma\times\cdots\times\Sigma}$ defined in Eq.~\eqref{eq:def.hankel.tensor}:
$$\H_{\Pref,\Suff} = \tenmatgen{\Ht^{(2l)}}{l,l}.$$
Indeed, for $u=u_1u_2\cdots u_l\in\Pref=\Sigma^l$ and  $v=v_1v_2\cdots v_l\in\Suff=\Sigma^l$ we have
$$(\H_{\Pref,\Suff})_{u,v}=f(uv)=f(u_1u_2\cdots u_lv_1v_2\cdots v_l)=\Ht^{(2l)}_{u_1,u_2,\cdots, u_l,v_1,v_2,\cdots, v_l}.$$
Using the same argument, one can easily show that $\hb_\Pref=\hb_\Suff=\vectorize{\Ht^{(l)}}$.
Lastly, a similar observation can be done for the Hankel sub-blocks $\H_{\Pref,\Suff}^\sigma$ for each $\sigma\in\Sigma$: all of these sub-blocks are slices of the Hankel tensor $\Ht^{(2l+1)}$. Indeed, for any $u=u_1u_2\cdots u_l\in\Pref=\Sigma^l$ and  $v=v_1v_2\cdots v_l\in\Suff=\Sigma^l$ we have
$$(\H^\sigma_{\Pref,\Suff})_{u,v}=f(u\sigma v)=f(u_1\cdots u_l\sigma v_1\cdots v_l)=\Ht^{(2l+1)}_{u_1,\cdots, u_l,\sigma,v_1,\cdots, v_l}$$
from which it follows that
$$\H_{\Pref,\Suff}^\sigma = \left(\tenmatgen{\Ht^{(2l+1)}}{l,1,l}\right)_{:,\sigma,:}\text{ for all }\sigma\in\Sigma.$$

Thus, in the case where $\Pref=\Suff=\Sigma^l$, all the sub-blocks of the Hankel matrix one needs to estimate for the spectral learning algorithm are matricization of Hankel tensors of tensor train rank at most $n$~(where $n$ is the number of states of the target WFA). Let us assume for now that we have access to the true Hankel tensors $\Ht^{(l)}$, $\Ht^{(2l)}$ and $\Ht^{(2l+1)}$ given in the tensor train format~(how to estimate these Hankel tensors in the tensor train format from data will be discussed in Section~\ref{sec:leveraging.TT.Hankel}):
\begin{align*}
    \Ht^{(l)} &= \TT{\Gt^{(l)}_1,\cdots,\Gt^{(l)}_l}\\
    \Ht^{(2l)} &= \TT{\Gt^{(2l)}_1,\cdots,\Gt^{(2l)}_{2l}}\\
    \Ht^{(2l+1)} &= \TT{\Gt^{(2l+1)}_1,\cdots,\Gt^{(2l+1)}_{2l+1}}
\end{align*}
where all tensor train decompositions are of rank $n$.
We now show how the tensor train structure of the Hankel tensors can be leveraged to significantly improve the computational complexity of the spectral learning algorithm. Recall first that in the standard case, this complexity is in $\bigo{n|\Pref||\Suff| + n^2|\Pref||\Sigma|}$~(where the first term corresponds to the truncated SVD of the Hankel matrix, and the second one to computing the transition matrices $\A^\sigma)$, which  is equal to $\bigo{n|\Sigma|^{2l} + n^2|\Sigma|^{l+1}}$ when $\Pref=\Suff=\Sigma^l$. In contrast, we will show that if the Hankel tensors are given in the tensor train format, the complexity of the spectral learning algorithm can be reduced to $\bigo{n^3l|\Sigma|}$.

First observe that the tensor train decomposition of the Hankel tensor $\Ht^{(2l)}$ already gives us the  rank $n$ factorization of the Hankel matrix $\H_{\Pref,\Suff}=\tenmatgen{\Ht^{(2l)}}{l,l}$, which can easily be seen from the following tensor network:
$$\Ht^{(2l)}=\begin{tikzpicture}[baseline=-0.5ex]

\tikzset{tensor/.style = {minimum size = 0.4cm,shape = circle,thick,draw=black,fill=blue!60!green!40!white,inner sep = 0pt}, edge/.style   = {thick,line width=.4mm},every loop/.style={}}
\node[tensor,fill=blue!60!green!40!white] (G0) at (0,0) {};
\draw[edge] (G0) -- (0,-0.5);
\node[draw=none,right=0.2cm of G0] (G0right) {};
\draw[edge] (G0) -- (G0right);
\node[draw=none] at (0,0.7) {$\Gt^{(2l)}_1$};

\node[tensor,fill=blue!60!green!40!white] (G1) at (0.8,0) {};
\draw[edge] (G1) -- (0.8,-0.5);
\node[draw=none,left=0.2cm of G1] (G1left) {};
\draw[edge] (G1) -- (G1left);
\node[draw=none,right=0.2cm of G1] (G1right) {};
\draw[edge] (G1) -- (G1right);
\node[draw=none] at (0.8,0.7) {$\Gt^{(2l)}_2$};

\node[draw=none] (G2) at (1.6,0) {$\cdots$};
\node[draw=none] at (1.6,0.7) {$\cdots$};

\node[tensor,fill=blue!60!green!40!white] (G4) at (2.4000000000000004,0) {};
\draw[edge] (G4) -- (2.4000000000000004,-0.5);
\node[draw=none,left=0.2cm of G4] (G4left) {};
\draw[edge] (G4) -- (G4left);
\node[draw=none,right=0.2cm of G4] (G4right) {};
\draw[edge] (G4) -- (G4right);
\node[draw=none] at (2.4000000000000004,0.7) {$\Gt^{(2l)}_{l}$};

\draw[edge] (2.8000000000000003,0) -- (3.7,0);
\node[tensor,fill=blue!60!green!40!white] (G5) at (3.7,0) {};
\draw[edge] (G5) -- (3.7,-0.5);
\node[draw=none,left=0.2cm of G5] (G5left) {};
\draw[edge] (G5) -- (G5left);
\node[draw=none,right=0.2cm of G5] (G5right) {};
\draw[edge] (G5) -- (G5right);
\node[draw=none] at (3.7,0.7) {$\Gt^{(2l)}_{l+1}$};

\node[draw=none] (G6) at (4.5,0) {$\cdots$};
\node[draw=none] at (4.5,0.7) {$\cdots$};

\node[tensor,fill=blue!60!green!40!white] (G7) at (5.3,0) {};
\draw[edge] (G7) -- (5.3,-0.5);
\node[draw=none,left=0.2cm of G7] (G7left) {};
\draw[edge] (G7) -- (G7left);
\node[draw=none,right=0.2cm of G7] (G7right) {};
\draw[edge] (G7) -- (G7right);
\node[draw=none] at (5.3,0.7) {$\Gt^{(2l)}_{2l-1}$};

\node[tensor,fill=blue!60!green!40!white] (G8) at (6.1,0) {};
\draw[edge] (G8) -- (6.1,-0.5);
\node[draw=none,left=0.2cm of G8] (G8left) {};
\draw[edge] (G8) -- (G8left);
\node[draw=none] at (6.1,0.7) {$\Gt^{(2l)}_{2l}$};

\draw[decoration={brace,mirror},decorate,thick] (-0.3,-0.7) -- node[below=5pt] {\P} (2.8,-0.7);
\draw[decoration={brace,mirror},decorate,thick] (3.2,-0.7) -- node[below=5pt] {\S} (6.3,-0.7);
\node[draw=none] () at (0.4,0.2) {\textcolor{gray}{$n$}};
\node[draw=none] () at (3,0.2) {\textcolor{gray}{$n$}};
\node[draw=none] () at (5.7,0.2) {\textcolor{gray}{$n$}};
\end{tikzpicture} $$
More formally, this shows that $\H_{\Pref,\Suff}=\P\S$ with $\P=\tenmatgen{\TT{\Gt^{(2l)}_1,\cdots,\Gt^{(2l)}_{l},\I}}{l,1}$ and $\S=\tenmatgen{\TT{\I,\Gt^{(2l)}_{l+1},\cdots,\Gt^{(2l)}_{2l}}}{1,l}$. The remaining step of the spectral learning algorithm consists in computing the pseudo-inverse of $\P$ and $\S$ and performing various matrix products involving the Hankel sub-blocks $\hb_\Pref$, $\hb_\Suff$ and  $\H_{\Pref,\Suff}^\sigma$ for each $\sigma\in\Sigma$. Observe that all the elements involved in these computations are tensors of tensor train rank at most $n$~(or matricizations of such tensors). It turns out that all these operations can be performed efficiently in the tensor tensor train format: the pseudo-inverses of $\P$ and $\S$ in the tensor train format can be computed in time $\bigo{n^3l|\Sigma|}$ and all the matrix products between $\P\pinv$ and $\S\pinv$ and the Hankel tensors can also be done in time $\bigo{n^3l|\Sigma|}$. Describing these tensor train computations in details go beyond the scope of this paper but these algorithms are well known in the tensor train and matrix product states communities. We refer the reader to~\citep{oseledets2011tensor} for efficient computations of matrix products in the tensor train format, and to~\citep{gelss2017tensor} and~\citep{klus2018tensor} for the computation of pseudo-inverse in the tensor train format.

We showed that in the case where $\Pref=\Suff=\Sigma^l$, the time complexity of the last two steps of the spectral learning algorithm can be reduced from an exponential dependency on $l$ to a linear one. This is achieved by leveraging the tensor train structure of the Hankel sub-blocks. However, recall that the spectral learning algorithm is consistent~(i.e. guaranteed to return the target WFA from an infinite amount of training data) only if $\Pref$ and $\Suff$ form a complete basis, that is $\Pref$ and $\Suff$ are such that  $\rank(\tenmatpar{\Hten_{\Pref,\Suff}}{1}) = \rank(\H)$. In the case where $\Pref=\Suff=\Sigma^l$, this condition is equivalent to $\rank(\tenmatgen{\Hten^{(2l)}}{l,l}) = \rank(\H)$. But it is not the case that for any function computed by a WFA, there exists an integer $l$ such that $\Pref=\Suff=\Sigma^l$ form a complete basis. Indeed, consider for example the function $f$ on the alphabet $\{a,b\}$ defined by $f(x)=1$ if $x=aa$ and $0$ otherwise. One can easily show that there exists a minimal WFA with 3 states computing $f$. However, it is easy to check that $\rank(\tenmatgen{\Hten^{(2l)}}{l,l})$ is equal to $1$ for $l=1$ and to $0$ for any other value of $l$. This implies that not all functions can be consistently recovered from training data using the efficient spectral learning algorithm we propose.

Luckily, this caveat can be addressed using a simple workaround. For any function $f:\Sigma^*\to \Rbb$, one can define a new alphabet $\tilde{\Sigma}=\Sigma\cup\{\underline{\lambda}\}$ where $\underline{\lambda}$ is a new symbol not in $\Sigma$ which will be treated as the empty string. One can then extend $f$ to $\tilde{f}:\tilde{\Sigma}^*\to\R$ naturally by ignoring the new symbol $\underline{\lambda}$, e.g. $\tilde{f}(\underline{\lambda}ab\underline{\lambda}c)=f(abc)$. Let $\H \in\R^{\Sigma^*\times\Sigma^*}$, $\tilde{\H}\in\R^{\tilde{\Sigma}^*\times\tilde{\Sigma}^*}$, $\Ht^{(2l)}\in\R^{\Sigma\times\cdots\times\Sigma}$ and $\tilde{\Ht}^{(2l)}\in\R^{\tilde{\Sigma}\times\cdots\times\tilde{\Sigma}}$ be the Hankel matrix and tensors of $f$ and $\tilde{f}$. Then, one can show that if $f$ can be computed by a WFA with $n$ states, there always exists an integer $l$ such that $\rank(\tenmatgen{\tilde{\Hten}^{(2l)}}{l,l})=\rank(\H)=n$. Indeed, in contrast with the Hankel tensor $\Ht^{(2l)}$ which only contains the values of $f$ on sequences of length exactly $2l$, the Hankel tensors $\tilde{\Ht}^{(2l)}$ contains the values of $f$ on all sequences of length smaller than or equal to $2l$.
One potential workaround would consist in estimating the Hankel sub-blocks of $\tilde{f}$ from data generated by $f$ and perform steps $3$ and $4$ of the spectral learning to recover the parameters of a WFA computing $\tilde{f}$. The transition matrix associated with the new symbol $\underline{\lambda}$ can be discarded to obtain the parameters of a WFA estimating $f$~(note that since the spectral learning algorithm is consistent, the transition matrix associated with the new symbol $\underline{\lambda}$ estimated from data is guaranteed to converge to the identity matrix as the training data increases). 

In practice, to estimate a Hankel tensor of length $L$, one could append every sequence in the dataset that is of length $L$ or smaller than $L$ with $\underline{\lambda}$ until it reaches length $L$ and then perform the standard Hankel recovery routine. It is worth mentioning that we did not have to use this workaround for any of the experiments presented in Section~\ref{sec:xp}. More importantly, one can show that if the parameters of a 2-RNN are drawn randomly then the workaround discussed above is not necessary~(i.e., one can consistently recover a random 2-RNN from data using the learning algorithm we propose), which is shown in the following proposition. 

\begin{proposition}
Let $\Acal = \langle \szero, \At, \sinf \rangle $ be a 2-RNN with $n$ states whose parameters are randomly drawn from a continuous distribution~(w.r.t. the Lebesgue measure) and let $\H\in\R^{\Sigma^*\times \Sigma^*}$ be its Hankel matrix. Then, with probability one, $\rank(\tenmatgen{\Hten^{(2l)}}{l,l}) = \rank(\H)$ for any $l$ such that $|\Sigma|^l\geq n$~(where $\Hten^{(2l)}$ is as defined in Eq.~\eqref{eq:def.hankel.tensor}). 
\end{proposition}

\begin{proof}
Let $\mat{F}_l\in\mathbb{R}^{\Sigma^l\times n}$ and $\mat{B}_l\in \mathbb{R}^{n\times \Sigma^l}$ be the forward and backward matrices of the random 2-RNN, that is the rows of $\mat{F}_l$ are $\szero^\top \At_{:.u_1,:}\cdots \At_{:.u_l,:}$ for $u_1\cdots u_l \in \Sigma^l$ and the columns of  $\mat{B}_l$ are $\At_{:.v_1,:}\cdots \At_{:.v_l,:}\sinf$ for $v_1\cdots v_l \in \Sigma^l$. 
Let $l$ be any integer such that $|\Sigma|^l \geq n$. 
We first show that both $\mat{F}_l$ and $\mat{B}_l$ are full rank with probability one.
Observe that $\det(\mat{F}_l^\top \mat{F}_l)$ is a polynomial of the 2-RNN parameters $\szero, \At$. 
Since a polynomial is either zero or non-zero almost everywhere~\citep{caron2005zero}, and since one can easily find a 2-RNN such that $\det(\mat{F}_l^\top \mat{F}_l)\neq 0$~(using the fact that $|\Sigma|^l \geq n$), it follows that 
$\det(\mat{F}_l^\top \mat{F}_l)$ is non-zero almost everywhere. Consequently, since the parameters $\szero$ and $\At$ are drawn from a continuous distribution,  $\det(\mat{F}_l^\top \mat{F}_l)\neq 0$ with probability one, i.e. $\mat{F}_l$ is of rank $n$ with probability one. With a similar argument, one can show that $\mat{B}_l$ is of rank $n$ with probability one.

To conclude, since $\tenmatgen{\Hten^{(2l)}}{l,l} = \mat{F}_l\mat{B}_l$, both $\mat{F}_l$ and $\mat{B}_l$ have rank $n$, and $|\Sigma|^l \geq n$, it follows that $\tenmatgen{\Hten^{(2l)}}{l,l}$ has rank $n$ with probability one.

\end{proof}


\section{Weighted Automata and Second-Order Recurrent Neural Networks}
\label{sec:WFA.2RNN.equivalence}

In this section, we present an equivalence result between weighted automata and second-order RNN with linear activation functions~(linear 2-RNN). This result rigorously formalizes the idea that WFA are linear RNN. 

Recall that a 2-RNN $R=(\szero,\Aten,\vvsinf)$ maps any sequence of inputs $\x_1,\cdots,\x_k\in\R^d$ to
a sequence of outputs $\y_1,\cdots,\y_k\in\R^p$ defined for any $ t=1,\cdots,k$ by
\begin{equation}\label{eq:2RNN.definition}
\y_t = z_2(\vvsinf\h_t) \text{ with }\h_t = z_1(\Aten\ttv{1}\h_{t-1}\ttv{2}\x_t)
\end{equation}
where $z_1:\R^n\to\R^n$ and $z_2:\R^p\to\R^p$ are activation functions.
We think of a 2-RNN as computing
a function $f_R:(\R^d)^*\to\R^p$ mapping each input sequence $\x_1,\cdots,\x_k$ to the corresponding final output $\y_k$.
While $z_1$ and $z_2$ are usually non-linear
component-wise functions, we consider here the case where both $z_1$ and $z_2$ are the identity, and we refer to
the resulting model as a \emph{linear 2-RNN}.

Observe that for a linear 2-RNN $R$, the function $f_R$ is multilinear in the sense that, for any integer $l$, its restriction to the domain $(\R^d)^l$ is
multilinear. Another useful observation is that linear 2-RNN are invariant under change of basis: for any invertible matrix
$\P$, the linear 2-RNN $\tilde{M}=(\P\invtop\h_0,\Aten\ttm{1}\P\ttm{3}\P\invtop,\P\vvsinf)$ is such that $f_{\tilde{M}}=f_M$. 

One can easily show that the computation of the linear 2-RNN $R=(\szero,\Aten,\vvsinf)$ boils down to the following tensor network~(see proof of Theorem~\ref{thm:2RNN-vvWFA}):
$$f_R(\x_1,\x_2,\cdots, \x_k)=
\begin{tikzpicture}[baseline=-0.5ex]
\tikzset{tensor/.style = {minimum size = 0.4cm,shape = circle,thick,draw=black,fill=blue!60!green!40!white,inner sep = 0pt}, edge/.style   = {thick,line width=.4mm},every loop/.style={}}

\tikzset{tensor/.style = {minimum size = 0.4cm,shape = circle,thick,draw=black,fill=blue!60!green!40!white,inner sep = 0pt}}
\node[tensor,fill=blue!60!green!40!white] (G0) at (0,0) {};
\node[draw=none,right=0.2cm of G0] (G0right) {};
\draw[edge] (G0) -- (G0right);
\node[draw=none] at (0,0.5) {$\szero$};

\node[tensor,fill=blue!60!green!40!white] (G1) at (0.8,0) {};
\draw[edge] (G1) -- (0.8,-0.5);
\node[draw=none,left=0.2cm of G1] (G1left) {};
\draw[edge] (G1) -- (G1left);
\node[draw=none,right=0.2cm of G1] (G1right) {};
\draw[edge] (G1) -- (G1right);
\node[draw=none] at (0.8,0.5) {$\At$};
\node[tensor,fill=green!60!red!40!white,minimum size = 0.5cm] at (0.8,-0.7) {$\x_1$};

\node[tensor,fill=blue!60!green!40!white] (G2) at (1.6,0) {};
\draw[edge] (G2) -- (1.6,-0.5);
\node[draw=none,left=0.2cm of G2] (G2left) {};
\draw[edge] (G2) -- (G2left);
\node[draw=none,right=0.2cm of G2] (G2right) {};
\draw[edge] (G2) -- (G2right);
\node[draw=none] at (1.6,0.5) {$\At$};
\node[tensor,fill=green!60!red!40!white,minimum size = 0.5cm] at (1.6,-0.7) {$\x_2$};

\node[draw=none] (G3) at (2.4000000000000004,0) {$\cdots$};
\node[draw=none] at (2.4000000000000004,0.5) {$\cdots$};

\node[tensor,fill=blue!60!green!40!white] (G5) at (3.2,0) {};
\draw[edge] (G5) -- (3.2,-0.5);
\node[draw=none,left=0.2cm of G5] (G5left) {};
\draw[edge] (G5) -- (G5left);
\node[draw=none,right=0.2cm of G5] (G5right) {};
\draw[edge] (G5) -- (G5right);
\node[draw=none] at (3.2,0.5) {$\At$};
\node[tensor,fill=green!60!red!40!white,minimum size = 0.5cm] at (3.2,-0.7) {$\x_k$};

\node[tensor,fill=blue!60!green!40!white] (G6) at (4.0,0) {};
\node[draw=none,left=0.2cm of G6] (G6left) {};
\draw[edge] (G6) -- (G6left);
\draw[edge] (G6) -- (4.5,0);
\node[draw=none] at (4.0,0.5) {$\vvsinf$};
\end{tikzpicture}
$$

This computation is very similar, not to say equivalent, to the computation of a WFA $A=\vvwa$. Indeed, as we showed in the previous section, by stacking the transition matrices $\{\A^\sigma\}_{\sigma\in\Sigma}$ into a third order tensor $\At\in\R^{n\times\Sigma\times n}$ the computation of the WFA $A$ can be written as
\begin{align*}
f(\sigma_1\sigma_2\cdots \sigma_k)
&=
\begin{tikzpicture}[baseline=-0.5ex]
\tikzset{tensor/.style = {minimum size = 0.4cm,shape = circle,thick,draw=black,fill=blue!60!green!40!white,inner sep = 0pt}, edge/.style   = {thick,line width=.4mm},every loop/.style={}}

\node[tensor] (G0) at (0,0) {};
\node[draw=none,right=0.2cm of G0] (G0right) {};
\draw[edge] (G0) -- (G0right);
\node[draw=none] at (0,0.5) {$\szero$};

\draw[edge] (0.39999999999999997,0) -- (1.5999999999999999,0);
\node[tensor] (G1) at (1.0,0) {};
\node[draw=none,left=0.2cm of G1] (G1left) {};
\draw[edge] (G1) -- (G1left);
\node[draw=none,right=0.2cm of G1] (G1right) {};
\draw[edge] (G1) -- (G1right);
\node[draw=none] at (1.0,0.5) {$\A^{\sigma_1}$};

\node[tensor] (G2) at (2.0,0) {};
\node[draw=none,left=0.2cm of G2] (G2left) {};
\draw[edge] (G2) -- (G2left);
\node[draw=none,right=0.2cm of G2] (G2right) {};
\draw[edge] (G2) -- (G2right);
\node[draw=none] at (2.0,0.5) {$\A^{\sigma_2}$};

\node[draw=none] (G3) at (2.8,0) {$\cdots$};
\node[draw=none] at (2.8,0.5) {$\cdots$};

\node[tensor] (G4) at (3.5999999999999996,0) {};
\node[draw=none,left=0.2cm of G4] (G4left) {};
\draw[edge] (G4) -- (G4left);
\node[draw=none,right=0.2cm of G4] (G4right) {};
\draw[edge] (G4) -- (G4right);
\node[draw=none] at (3.5999999999999996,0.5) {$\A^{\sigma_{k-1}}$};

\draw[edge] (3.999999999999999,0) -- (5.2,0);
\node[tensor] (G5) at (4.6,0) {};
\node[draw=none,left=0.2cm of G5] (G5left) {};
\draw[edge] (G5) -- (G5left);
\node[draw=none,right=0.2cm of G5] (G5right) {};
\draw[edge] (G5) -- (G5right);
\node[draw=none] at (4.6,0.5) {$\A^{\sigma_k}$};

\node[tensor] (G6) at (5.6,0) {};
\node[draw=none,left=0.2cm of G6] (G6left) {};
\draw[edge] (G6) -- (G6left);
\draw[edge] (G6) -- (6.1,0);
\node[draw=none] at (5.6,0.5) {$\vvsinf$};

\end{tikzpicture}\\
&=
\begin{tikzpicture}[baseline=-0.5ex]
\tikzset{tensor/.style = {minimum size = 0.4cm,shape = circle,thick,draw=black,fill=blue!60!green!40!white,inner sep = 0pt}, edge/.style   = {thick,line width=.4mm},every loop/.style={}}

\tikzset{tensor/.style = {minimum size = 0.4cm,shape = circle,thick,draw=black,fill=blue!60!green!40!white,inner sep = 0pt}}
\node[tensor,fill=blue!60!green!40!white] (G0) at (0,0) {};
\node[draw=none,right=0.2cm of G0] (G0right) {};
\draw[edge] (G0) -- (G0right);
\node[draw=none] at (0,0.5) {$\szero$};

\node[tensor,fill=blue!60!green!40!white] (G1) at (0.8,0) {};
\draw[edge] (G1) -- (0.8,-0.5);
\node[draw=none,left=0.2cm of G1] (G1left) {};
\draw[edge] (G1) -- (G1left);
\node[draw=none,right=0.2cm of G1] (G1right) {};
\draw[edge] (G1) -- (G1right);
\node[draw=none] at (0.8,0.5) {$\At$};
\node[draw=none] at (0.8,-0.7) {$\sigma_1$};

\node[tensor,fill=blue!60!green!40!white] (G2) at (1.6,0) {};
\draw[edge] (G2) -- (1.6,-0.5);
\node[draw=none,left=0.2cm of G2] (G2left) {};
\draw[edge] (G2) -- (G2left);
\node[draw=none,right=0.2cm of G2] (G2right) {};
\draw[edge] (G2) -- (G2right);
\node[draw=none] at (1.6,0.5) {$\At$};
\node[draw=none] at (1.6,-0.7) {$\sigma_2$};

\node[draw=none] (G3) at (2.4000000000000004,0) {$\cdots$};
\node[draw=none] at (2.4000000000000004,0.5) {$\cdots$};

\node[tensor,fill=blue!60!green!40!white] (G4) at (3.2,0) {};
\draw[edge] (G4) -- (3.2,-0.5);
\node[draw=none,left=0.2cm of G4] (G4left) {};
\draw[edge] (G4) -- (G4left);
\node[draw=none,right=0.2cm of G4] (G4right) {};
\draw[edge] (G4) -- (G4right);
\node[draw=none] at (3.2,0.5) {$\At$};
\node[draw=none] at (3.2,-0.7) {$\sigma_{k-1}$};

\node[tensor,fill=blue!60!green!40!white] (G5) at (4.0,0) {};
\draw[edge] (G5) -- (4.0,-0.5);
\node[draw=none,left=0.2cm of G5] (G5left) {};
\draw[edge] (G5) -- (G5left);
\node[draw=none,right=0.2cm of G5] (G5right) {};
\draw[edge] (G5) -- (G5right);
\node[draw=none] at (4.0,0.5) {$\At$};
\node[draw=none] at (4.0,-0.7) {$\sigma_k$};

\node[tensor,fill=blue!60!green!40!white] (G6) at (4.8,0) {};
\node[draw=none,left=0.2cm of G6] (G6left) {};
\draw[edge] (G6) -- (G6left);
\draw[edge] (G6) -- (5.3,0);
\node[draw=none] at (4.8,0.5) {$\vvsinf$};
\end{tikzpicture}
\end{align*}
Thus, if we restrict the input vectors of a linear 2-RNN to be one-hot encoding~(i.e. vectors from the canonical basis), the two models are strictly equivalent.

These observations unravel a fundamental connection between vv-WFA and linear 2-RNN: vv-WFA and
linear 2-RNN are expressively equivalent for representing functions defined over sequences of
discrete symbols. Moreover, both models have the same capacity in the sense that there is a direct
correspondence between the hidden units of a linear 2-RNN and the states of a vv-WFA computing the same function. More formally, we have the following theorem. 

\begin{theorem}\label{thm:2RNN-vvWFA}
Any function that can be computed by a vv-WFA with $n$ states can be computed by a linear 2-RNN with $n$ hidden units.
Conversely, any function that can be computed by a linear 2-RNN with $n$ hidden units on sequences of one-hot vectors~(\ie canonical basis 
vectors) can be computed by a WFA with $n$ states.

More precisely, the WFA $A=\vvwa$ with $n$ states and the linear 2-RNN $M=(\szero,\Aten,\vvsinf)$ with
$n$ hidden units, where $\Aten\in\R^{n\times \Sigma \times n}$ is defined by $\Aten_{:,\sigma,:}=\A^\sigma$ for all $\sigma\in\Sigma$, are
such that
$f_A(\sigma_1\sigma_2\cdots\sigma_k) = f_M(\x_1,\x_2,\cdots,\x_k)$ for all sequences of input symbols $\sigma_1,\cdots,\sigma_k\in\Sigma$,
where for each $i\in[k]$ the input vector $\x_i\in\R^\Sigma$ is
the one-hot encoding of the symbol $\sigma_i$.
\end{theorem}
\begin{proof}
We first show by induction on $k$ that, for any sequence $\sigma_1\cdots\sigma_k\in\Sigma^*$, the hidden state $\h_k$ computed by $M$~(see
Eq.~\eqref{eq:2RNN.definition})
on the corresponding one-hot encoded sequence $\x_1,\cdots,\x_k\in\R^d$
satisfies $\h_k = (\A^{\sigma_1}\cdots\A^{\sigma_k})^\top\szero$. The case $k=0$
is immediate. Suppose the result true for sequences of length up to $k$. One can check easily check that $\Aten\ttv{2}\x_i = \A^{\sigma_i}$
for any index $i$. Using the induction hypothesis it then follows that
\begin{align*}
\h_{k+1} &= \Aten \ttv{1}\h_k \ttv{2} \x_{k+1} = \A^{\sigma_{k+1}}\ttv{1} \h_k = (\A^{\sigma_{k+1}})^\top \h_k\\
&= (\A^{\sigma_{k+1}})^\top (\A^{\sigma_1}\cdots\A^{\sigma_k})^\top\szero = (\A^{\sigma_1}\cdots\A^{\sigma_{k+1}})^\top\szero .
\end{align*} 
To conclude, we thus have
\begin{equation*}
f_M(\x_1,\x_2,\cdots,\x_k) = \vvsinf\h_{k} = \vvsinf(\A^{\sigma_1}\cdots\A^{\sigma_{k}})^\top\szero = f_A(\sigma_1\sigma_2\cdots\sigma_k).\qedhere
\end{equation*}
\end{proof}
This result first implies that linear 2-RNN defined over sequences of discrete symbols~(using one-hot encoding) \emph{can be provably learned using the spectral
learning algorithm for WFA/vv-WFA}; indeed, these  algorithms have been proved to compute consistent estimators.
Let us stress again that, contrary to the case of feed-forward architectures, learning recurrent networks with linear activation functions is not a trivial task.
Furthermore, Theorem~\ref{thm:2RNN-vvWFA} reveals that linear 2-RNN are a natural generalization of classical weighted automata  to functions
defined over sequences of continuous vectors~(instead of discrete symbols). This spontaneously raises the question of whether the spectral learning algorithms
for WFA and vv-WFA can be extended to the general setting of linear 2-RNN; we show that the answer is in the positive in the next section.

\section{Spectral Learning of Continuous Weighted Automata}
\label{sec:spec.learn.2RNN}

In this section, we extend the learning algorithm for vv-WFA to linear 2-RNN, thus at the same time addressing the limitation of the spectral learning algorithm to discrete inputs and  providing the first consistent learning algorithm for linear second-order RNN.


\subsection{Recovering 2-RNN from Hankel Tensors}\label{subsec:SL-2RNN}
We first present an identifiability result showing how one can recover a linear 2-RNN computing a function $f:(\R^d)^*\to \R^p$ from observable tensors extracted from some Hankel tensor
associated with $f$. Intuitively, we obtain this result by reducing the problem to the one of learning a vv-WFA. This is done by considering the restriction of
$f$ to canonical basis vectors; loosely speaking, since the domain of this restricted function is isomorphic to $[d]^*$, this allows us to fall back onto the 
setting of sequences of discrete symbols.

It is not straightforward how the notion of Hankel matrix can be extended to a function $f:(\R^d)^*\to\R^p$ taking sequences of \emph{continuous} vectors as input. One natural way to proceed is to consider how $f$ acts on sequences of vectors from the canonical basis. Given a function  $f:(\R^d)^*\to\R^p$, we define its Hankel tensor $\Hten_f\in \R^{[d]^* \times [d]^* \times p}$ by
$$(\Hten_f)_{i_1\cdots i_s,  j_1\cdots j_t,:} = f(\e_{i_1},\cdots,\e_{i_s},\e_{j_1},\cdots,\e_{j_t}),$$ 
for all $i_1,\cdots,i_s,j_1,\cdots,j_t\in [d]$, which is infinite in two
of its modes. It is easy to see that $\Hten_f$ is also the Hankel tensor associated with the function $\tilde{f}:[d]^* \to \R^p$ mapping any
sequence $i_1i_2\cdots i_k\in[d]^*$ to $f(\e_{i_1},\cdots,\e_{i_k})$. Moreover, in the special case where $f$ can be computed by a linear 2-RNN, one
can use the multilinearity of $f$ to show that $$f(\x_1,\cdots,\x_k) = \sum_{i_1,\cdots,i_k = 1}^d (\x_1)_{i_1}\cdots(\x_l)_{i_k} \tilde{f}(i_1\cdots i_k).$$
This gives us some intuition on how one could learn $f$ by learning a vv-WFA computing $\tilde{f}$ using the spectral learning algorithm. 
That is, assuming access to the sub-blocks of the Hankel tensor $\Hten$ for a complete basis of prefixes and suffixes $\Pref,\Suff\subseteq [d]^*$, the spectral learning algorithm can be used to recover a vv-WFA computing $\tilde{f}$ and consequently a linear 2-RNN computing $f$ using Theorem~\ref{thm:2RNN-vvWFA}.


We now state the main result of this section, showing that a (minimal) linear 2-RNN computing a function $f:(\R^d)^*\to\R$ can be exactly recovered from sub-blocks of the Hankel tensor $\Hten_f$. For the sake of clarity, we present the learning algorithm for the particular case where there exists an $L$ such that 
the prefix and suffix sets consisting of all sequences of length $L$, that is $\Pref = \Suff = 
[d]^L$, forms a complete basis for $\tilde{f}$~(\ie the sub-block $\Hten_{\Pref,\Suff}\in\R^{[d]^L\times [d]^L\times p}$ of the Hankel tensor $\Hten_f$ is
such that $\rank(\tenmatpar{\Hten_{\Pref,\Suff}}{1}) = \rank(\tenmatpar{\Hten_f}{1})$). As discussed in Section~\ref{sec:spectral.learning.TT}, such an integer $L$ does not always exist even when the underlying function $f$ can be computed by a linear 2-RNN. However, the workaround described at the end of Section~\ref{sec:spectral.learning.TT} can be used here as well to extend this theorem to the case of any function $f$ that can be computed by a linear 2-RNN.

The following theorem can be seen as a reformulation of the classical spectral learning theorem using the low rank Hankel tensors  $\Ht^{(l)}$ introduced in Section~\ref{sec:TT.struct.hankel}. In the case of a continuous function $f:(\R^d)^*\to\R$, for any integer $l$, the finite tensor 
$\Hten^{(l)}_f\in \R^{ d\times \cdots \times d\times p}$ of order $l+1$ is defined by
$$  (\Hten^{(l)}_f)_{i_1,\cdots,i_l,:} = f(\e_{i_1},\cdots,\e_{i_l}) \ \ \ \text{for all } i_1,\cdots,i_l\in [d].$$
Observe that for any integer $l$, the tensor $\Hten^{(l)}_f$ can be obtained by reshaping a finite sub-block of the Hankel tensor $\Hten_f$. 

\begin{theorem}\label{thm:2RNN-SL}
Let $f:(\R^d)^*\to \R^p$ be a function computed by a minimal linear $2$-RNN  with $n$ hidden units and let
$L$ be an integer such that $$\rank(\tenmatgen{\Hten^{(2L)}_f}{L,L+1}) = n.$$

Then, for any $\P\in\R^{d^L\times n}$ and $\S\in\R^{n\times d^Lp}$ such that $$\tenmatgen{\Hten^{(2L)}_f}{L,L+1} = \P\S,$$ the
linear 2-RNN $R=(\szero,\Aten,\vvsinf)$ defined by

$$\szero = (\S\pinv)^\top\tenmatgen{\Hten^{(L)}_f}{L+1}, \ \ \ \ \vvsinf^\top = \P\pinv\tenmatgen{\Hten^{(L)}_f}{L,1}
$$
$$\Aten = (\tenmatgen{\Hten^{(2L+1)}_f}{L,1,L+1})\ttm{1}\P\pinv\ttm{3}(\S\pinv)^\top$$
is a minimal linear $2$-RNN computing $f$.
\end{theorem}

\begin{proof}
Let $\P\in\R^{d^L\times n}$ and $\S\in\R^{n\times d^Lp}$ be such that $\tenmatgen{\Hten^{(2L)}_f}{L,L+1} = \P\S$ and let $R^\star=(\szero,\Aten,\vvsinf)$ be a minimal linear 2-RNN computing $f$.
Define the tensors 
$$\Pten^\star = \TT{\Aten^\star\ttv{1}\szero^\star, \underbrace{\Aten^\star, \cdots, \Aten^\star}_{L-1\text{ times}}, \I_n}\in\R^{d\times\cdots\times d\times n}$$
and
$$\Sten^\star = \TT{\I_n,\underbrace{\Aten^\star, \cdots, \Aten^\star}_{L\text{ times}}, \vvsinf^\star}\in\R^{n\times d\times\cdots\times d\times p}$$ 
of order $L+1$ and $L+2$ respectively, and let $\P^\star = \tenmatgen{\Pten^\star}{l,1} \in\R^{d^L\times n}$ 
and $\S^\star = \tenmatgen{\Sten^\star}{1,L+1}  \in\R^{n\times d^Lp}$. Using the identity $\Hten^{(l)}_f = \TT{\Aten\ttv{1}\szero, \underbrace{\Aten, \cdots, \Aten}_{l-1\text{ times}}, \vvsinf^\top}$ for
any $l$, one can easily check the following identities~(see also Section~\ref{sec:TT.struct.hankel}):
\begin{gather*}
\tenmatgen{\Hten^{(2L)}_f}{L,L+1} = \P^\star\S^\star,\ \  \ \ \tenmatgen{\Hten^{(2L+1)}_f}{L,1,L+1}= \Aten^\star \ttm{1} \P^\star \ttm{3} (\S^\star)^\top,\\
\tenmatgen{\Hten^{(L)}_f}{L,1} = \P^\star(\vvsinf^\star)^\top, \ \ \ \ \ \ \
\tenmatgen{\Hten^{(L)}_f}{L+1} = (\S^\star)^\top\szero.
\end{gather*}

Let $\M = \P\pinv\P^\star$. We will show that $\szero = \M\invtop\szero^\star$, $\Aten = \Aten^\star \ttm{1}\M\ttm{3}\M\invtop$ and
$\vvsinf = \M\vvsinf^\star$, which will entail the results since linear 2-RNN are invariant under change of basis. First observe that $\M\inv = \S^\star\S\pinv$. Indeed,
we have
$\P\pinv\P^\star\S^\star\S\pinv = \P\pinv\tenmatgen{\Hten^{(2L)}_f}{L,L+1}\S\pinv = \P\pinv\P\S\S\pinv = \I$
where we used the fact that $\P$~(resp. $\S$) is of full column rank~(resp. row rank) for the last equality. 

The following derivations then follow from basic tensor algebra:
\begin{align*}
\szero 
&= 
(\S\pinv)^\top\tenmatgen{\Hten^{(L)}_f}{L+1} 
=
(\S\pinv)^\top (\S^\star)^\top\szero
=
(\S^\star\S\pinv)^\top
=
\M\invtop\szero^\star,\\
\ \\
\Aten 
&= 
(\tenmatgen{\Hten^{(2L+1)}_f}{L,1,L+1})\ttm{1}\P\pinv\ttm{3}(\S\pinv)^\top\\
&=
(\Aten^\star \ttm{1} \P^\star \ttm{3} (\S^\star)^\top)\ttm{1}\P\pinv\ttm{3}(\S\pinv)^\top\\
&=
\Aten^\star \ttm{1} \P\pinv\P^\star \ttm{3} (\S^\star\S\pinv)^\top =  \Aten^\star \ttm{1}\M\ttm{3}\M\invtop,\\
\ \\
\vvsinf^\top 
&= 
\P\pinv\tenmatgen{\Hten^{(L)}_f}{L,1}
=
\P\pinv\P^\star(\vvsinf^\star)^\top
=
\M\vvsinf^\star,
\end{align*}
which concludes the proof.
\end{proof}
Observe that such an integer $L$ exists under the assumption that $\Pcal = \Scal = 
[d]^L$ forms a complete basis for $\tilde{f}$.
It is also worth mentioning that a necessary condition for $\rank(\tenmatgen{\Hten^{(2L)}_f}{L,L+1}) = n$ is that
$d^L\geq n$, \ie $L$ must be of the order $\log_d(n)$.

\subsection{Hankel Tensors Recovery from Linear Measurements}\label{subsec:Hankel.tensor.recovery}

We showed in the previous section that, given the Hankel tensors $\Hten^{(L)}_f$, $\Hten^{(2L)}_f$ and $\Hten^{(2L+1)}_f$, one can recover 
a linear 2-RNN computing $f$ if it exists. This first implies that the class of functions that can be computed by linear 2-RNN is learnable  in Angluin's
exact learning model~\citep{angluin1988queries} where one has access to an oracle that can answer membership queries~(\eg \textit{what is the value computed by the target $f$ 
on~$(\x_1,\cdots,\x_k)$?}) and  equivalence queries~(\eg \textit{is the current hypothesis $h$ equal to the target $f$?}). While this fundamental result is 
of significant theoretical interest, assuming access to such an oracle is unrealistic. In this section, we show that a stronger learnability result can be obtained in a more realistic setting,
where we  only
assume access to randomly generated input/output examples $((\x^{(i)}_1,\x_2^{(i)},\cdots,\x_l^{(i)}),\y^{(i)})\in(\R^d)^*\times\R^p$ where $\y^{(i)} = f(\x^{(i)}_1,\x_2^{(i)},\cdots,\x_l^{(i)})$.

The key observation is that such an  example $((\x^{(i)}_1,\x_2^{(i)},\cdots,\x_l^{(i)}),\y^{(i)})$ can be seen as a \emph{linear measurement} of the 
Hankel tensor $\Hten^{(l)}$. Indeed, let $f$ be a function computed by a linear 2-RNN. Using the multilinearity of $f$ we have
\begin{align*}
f(\x_1,\x_2,\cdots,\x_l) 
&=
f\left(\sum_{i_1} (\x_1)_{i_1}\eb_{i_1}, \sum_{i_2} (\x_2)_{i_2}\eb_{i_2},\cdots,\sum_{i_l} (\x_l)_{i_l}\eb_{i_l}\right)\\
&=
\sum_{i_1,\cdots,i_l} (\x_1)_{i_1}\cdots (\x_l)_{i_l} f( \eb_{i_1}, \cdots,\eb_{i_l})  \\
&=
\sum_{i_1,\cdots,i_l} (\x_1)_{i_1}\cdots (\x_l)_{i_l} \Ht^{(l)}_{i_1,\cdots,i_l}  \\
&=
\Hten^{(l)}_f \ttv{1} \x_1 \ttv{2} \cdots \ttv{l} \x_l \\
&= 
\tenmatgen{\Hten^{(l)}}{l,1}^\top (\x_1\kron\cdots \kron \x_{l})
\end{align*}
where $(\eb_1,\cdots,\eb_l)$ denotes the canonical basis of $\R^l$. 
It follows that  each input/output example $((\x^{(i)}_1,\x_2^{(i)},\cdots,\x_l^{(i)}),\y^{(i)})$ constitutes a linear measurement of $\Ht^{(l)}$:
\begin{align*}
\y^{(i)} 
= 
\tenmatgen{\Hten^{(l)}}{l,1}^\top (\x^{(i)}_1\kron\cdots \kron \x_{l}^{(i)})
=
\tenmatgen{\Hten^{(l)}}{l,1}^\top \x^{(i)}
\end{align*}
where $\x^{(i)} := \x^{(i)}_1\kron\cdots \kron \x_{l}^{(i)}\in\R^{d^l}$. 
Hence,  by regrouping $N$ output examples $\y^{(i)}$ into
the matrix $\Ymat\in\R^{N\times p}$ and the corresponding input vectors $\x^{(i)}$ into the matrix $\X\in\R^{N\times d^l}$,
one can recover $\Hten^{(l)}$ by solving the linear system $\Ymat = \X\tenmatgen{\Hten^{(l)}}{l,1}$, which has a unique
solution whenever $\X$ is of full column rank. This simple estimation technique for the Hankel tensors allows us to design the first consistent learning algorithm for linear 2-RNN, which is summarized in Algorithm~\ref{alg:2RNN-SL}~(with the "Least-Squares" recovery method). More efficient recovery methods for the Hankel tensors will be discussed in the next section.
The following theorem shows that this learning algorithm is consistent. Its proof relies on the fact that
$\X$ will be of full column rank whenever $N\geq d^l$ and  the components of each $\x^{(i)}_j$ for $j\in[l],i\in[N]$
are drawn independently from a continuous distribution over $\R^{d}$~(w.r.t. the Lebesgue measure).

\begin{theorem}\label{thm:learning-2RNN}
Let $(\h_0,\Aten,\vvsinf)$ be a minimal linear 2-RNN with $n$ hidden units computing a function $f:(\R^d)^*\to \R^p$, and let $L$ be an integer
such that $\rank(\tenmatgen{\Hten^{(2L)}_f}{L,L+1}) = n$.

Suppose we have access to $3$ datasets
$$D_l = \{((\x^{(i)}_1,\x_2^{(i)},\cdots,\x_l^{(i)}),\y^{(i)}) \}_{i=1}^{N_l}\subset(\R^d)^l\times \R^p \text{ for } l\in\{L,2L,2L+1\}$$ 
where the entries of each $\x^{(i)}_j$ are drawn independently from the standard normal distribution and where each
$\y^{(i)} = f(\x^{(i)}_1,\x_2^{(i)},\cdots,\x_l^{(i)})$.

Then, if $N_l \geq d^l$ for $l =L,\ 2L,\ 2L+1$,
the linear 2-RNN $M$ returned by Algorithm~\ref{alg:2RNN-SL} with the least-squares method satisfies $f_M = f$ with probability one.
\end{theorem}

\begin{proof}
We just need to show for each $l\in \{L,2L,2L+1\}$ that, under the hypothesis of the Theorem, the Hankel tensors $\hat{\Hten}^{(l)}$ computed in line~\ref{alg.line.lst-sq} of
Algorithm~\ref{alg:2RNN-SL} are equal to the true Hankel tensors $\Hten^{(l)}$ with probability one. Recall that these tensors are computed by solving the least-squares
problem
$$\hat{\Hten}^{(l)} = \argmin_{T\in \R^{d\times\cdots\times d\times p}} \norm{\X\tenmatgen{\T}{l,1} - \Ymat}_F^2$$
where $\X\in\R^{N_l\times d_l}$ is the matrix  with rows $\x^{(i)}_1\kron\x_2^{(i)}\kron\cdots\kron\x_l^{(i)}$ for each $i\in[N_l]$. Since  $\X\tenmatgen{\Hten^{(l)}}{l,1} = \Ymat$ and  the solution
of the least-squares problem is unique as soon as $\X$ is of full column rank, we just need to show that this is the case with probability one
when the entries of the vectors $\x^{(i)}_j$ are drawn at random from a standard normal distribution. The result will  then directly follow
by applying Theorem~\ref{thm:2RNN-SL}.

We will show that the set 
$$\Scal = \{ (\x_1^{(i)},\cdots, \x_l^{(i)}) \mid \ i\in[N_l],\ dim(span(\{ \x^{(i)}_1\kron\x_2^{(i)}\kron\cdots\kron\x_l^{(i)} \})) < d^l\} $$
has Lebesgue measure $0$ in $((\R^d)^{l})^{N_l}\simeq \R^{dlN_l}$ as soon as $N_l \geq d^l$, which will imply that it has probability $0$ under any continuous probability, hence
the result. For any  $S=\{(\x_1^{(i)},\cdots, \x_l^{(i)})\}_{i=1}^{N_l}$, we denote  by $\X_S\in\R^{N_l\times d^l}$  the matrix  with rows $\x^{(i)}_1\kron\x_2^{(i)}\kron\cdots\kron\x_l^{(i)}$.
One can easily check that $S\in\Scal$ if and only if $\X_S$ is of rank strictly less than $d^l$, which is equivalent to the determinant of 
$\X_S^\top\X_S$ being equal to $0$. Since this determinant is a polynomial in the entries of the vectors $\x_j^{(i)}$, $\Scal$ is an algebraic
subvariety of $\R^{dlN_l}$.
It is then easy to check that the polynomial $det(\X_S^\top\X_S)$ is not uniformly 0 when $N_l \geq d^l$. Indeed, 
it suffices to choose the vectors $\x_j^{(i)}$ such that the family  $(\x^{(i)}_1\kron\x_2^{(i)}\kron\cdots\kron\x_l^{(i)})_{n=1}^{N_l}$ spans the whole space 
$\R^{d^l}$~(which is possible since we can choose arbitrarily any of the $N_l\geq d^l$ elements of this family),  hence the result. 
In conclusion, $\Scal$ is a proper algebraic subvariety of $\R^{dlN_l}$ and hence has Lebesgue
measure zero~\cite[Section 2.6.5]{federer2014geometric}.

\end{proof}

\begin{algorithm}[tb]
   \caption{\texttt{2RNN-SL}: Spectral Learning of linear 2-RNN }
   \label{alg:2RNN-SL}
\begin{algorithmic}[1]
   \REQUIRE Three training datasets $D_L,D_{2L},D_{2L+1}$ with input sequences of length $L$, $2L$ and $2L+1$ respectively, a \texttt{recovery\_method}, rank $R$, noise tolerance parameter $\varepsilon$~(for Nuclear Norm) and learning rate $\gamma$~(for IHT/TIHT/SGD).
   \FOR{$l\in\{L,2L,2L+1\}$}
       \STATE\label{alg.firstline.forloop} Use $D_l = \{((\x^{(i)}_1,\x_2^{(i)},\cdots,\x_l^{(i)}),\y^{(i)}) \}_{i=1}^{N_l}\subset(\R^d)^l\times \R^p$ to build $\X\in\R^{N_l\times d^l}$ with rows $\x^{(i)}_1\kron\x_2^{(i)}\kron\cdots\kron\x_l^{(i)}$ for $i\in[N_l]$ and  $\Ymat\in\R^{N_l\times p}$ with rows $\y^{(i)}$ for $i\in[N_l]$.
       \IF{\texttt{recovery\_method} = "Least-Squares"}
           \STATE\label{alg.line.lst-sq} $$\Hten^{(l)} = \displaystyle\argmin_{\T\in \R^{d\times\cdots\times d\times p}} \norm{\X\tenmatgen{\T}{l,1} - \Ymat}_F^2.$$
       \ELSIF{\texttt{recovery\_method} = "Nuclear Norm"}
           \STATE $$\Hten^{(l)} = \displaystyle\argmin_{\T\in \R^{d\times\cdots\times d\times p}} \norm{\tenmatgen{\T}{\ceil{l/2},l-\ceil{l/2} + 1}}_*\ \ \ \text{subject to } \|\X \tenmatgen{\T}{l,1}- \Ymat\|\leq \varepsilon.$$
           \label{alg.line.nucnorm}
       \ELSIF{\texttt{recovery\_method} = "IHT" \OR \texttt{recovery\_method} =  "TIHT" }\label{alg.iht.start}
           \STATE Initialize $\Hten^{(l)} \in \R^{d\times\cdots\times d\times p}$.
           \REPEAT\label{alg.line.iht.start}
               \STATE\label{alg.line.iht.gradient} 
               $\tenmatgen{\Hten^{(l)}}{l,1} \leftarrow \tenmatgen{\Hten^{(l)} }{l,1} + \gamma\X^\top(\Ymat - \X\tenmatgen{\Hten^{(l)} }{l,1})$
               \STATE $\Hten^{(l)} \leftarrow \texttt{project}(\Hten^{(l)},R)$~(using either SVD for IHT or TT-SVD for TIHT)
               \UNTIL{convergence}\label{alg.line.iht.end}\label{alg.iht.stop}
       
       \ELSIF{\texttt{recovery\_method} = "SGD" \OR \texttt{recovery\_method} =  "ALS"}
            \STATE Initialize all cores of the rank $R$ TT-decomposition $\Hten^{(l)} = \TT{\Gten^{(l)}_1,\cdots,\Gten^{(l)}_{l+1}}$.\\
           // \emph{Note that $\Hten^{(l)}$ is never explicitly constructed.}
           \REPEAT\label{alg.line.sgdals.start}
                \FOR{$i=1,\cdots,l+1$}
                    \STATE $\Gten^{(l)}_i\leftarrow\begin{cases}
                    \Gten^{(l)}_i - \gamma \nabla_{\Gten^{(l)}_i} \norm{\X\tenmatgen{\TT{\Gten^{(l)}_1,\cdots,\Gten^{(l)}_{l+1}}}{l,1} - \Ymat}_F^2&\text{for SGD} \\
                     \displaystyle\argmin_{\Gten^{(l)}_i} \norm{\X\tenmatgen{\TT{\Gten^{(l)}_1,\cdots,\Gten^{(l)}_{l+1}}}{l,1} - \Ymat}_F^2&\text{for ALS}
                    \end{cases}$
                \ENDFOR
           \UNTIL{convergence}\label{alg.line.sgdals.end}
       
       \ENDIF\label{alg.lastline.forloop}

   \ENDFOR
   
   \STATE\label{alg.line.svd} Let $\tenmatgen{\Hten^{(2L)}}{L,L+1} = \P\S$ be a rank $R$ factorization.
   \STATE Return the linear 2-RNN $(\h_0,\Aten,\vvsinf)$ where 
   \begin{align*}
    \hb_0\ &= (\S\pinv)^\top\tenmatgen{\Hten^{(L)}_f}{L+1},\ \ \ \ \vvsinf^\top = \P\pinv\tenmatgen{\Hten^{(L)}_f}{L,1}\\
    \Aten\ &= (\tenmatgen{\Hten^{(2L+1)}_f}{L,1,L+1})\ttm{1}\P\pinv\ttm{3}(\S\pinv)^\top
\end{align*}  
\end{algorithmic}
\end{algorithm}

A few remarks on this theorem are in order.   The first observation is that the $3$
datasets $D_L$, $D_{2L}$ and $D_{2L+1}$ do not need to be drawn independently from one another~(\eg the sequences in $D_{L}$ can
be prefixes of the sequences in $D_{2L}$ but it is not necessary). In particular, the result still holds when  the datasets $D_L$, $D_{2L}$ and $D_{2L+1}$ are constructed from a unique dataset 
$$S =\{((\x^{(i)}_1,\x_2^{(i)},\cdots,\x_T^{(i)}),(\y^{(i)}_1,\y^{(i)}_2,\cdots,\y^{(i)}_T)) \}_{i=1}^{N}$$
of input/output sequences with $T\geq 2L+1$, where $\y^{(i)}_t = f(\x^{(i)}_1,\x_2^{(i)},\cdots,\x_t^{(i)})$ for any $t\in[T]$.
Observe that having access to such input/output training sequences is not an unrealistic assumption: for example when training RNN for
language modeling the output $\y_t$ is the  conditional probability vector of the next symbol. Lastly, when the outputs $\y^{(i)}$ are noisy, one can solve the least-squares problem
$\norm{\Ymat - \X\tenmatgen{\Hten^{(l)}}{l,1}}^2_F$ to approximate the Hankel tensors; we will empirically evaluate this approach 
in Section~\ref{sec:xp} and we defer its theoretical analysis  in the noisy setting to future work.

\subsection{Leveraging the low rank structure of the Hankel tensors}
\label{sec:leveraging.TT.Hankel}
While the least-squares method is sufficient to obtain the theoretical guarantees of Theorem~\ref{thm:learning-2RNN}, it does not leverage
the low rank structure of the Hankel tensors $\Hten^{(L)}$, $\Hten^{(2L)}$ and $\Hten^{(2L+1)}$. We now  propose several alternative recovery
methods to leverage this structure, in order to improve both sample complexity and time complexity. The sample efficiency and running time of these methods will be assessed in a simulation study in Section~\ref{sec:xp}~(deriving improved sample
complexity guarantees using these methods is left for future work).

We first propose two alternatives to solving the least-squares problem  $\Ymat = \X\tenmatgen{\Hten^{(l)}}{l,1}$ that leverage the low matrix rank structure of the Hankel tensor. Indeed, knowing that $\tenmatgen{\Hten^{(l)}}{\ceil{l/2},l-\ceil{l/2} + 1}$ can be approximately low rank~(if the target function is computed by a WFA with a small number of states), one can achieve better sample complexity by taking into account the fact that the effective number of parameters needed to describe this matrix can be significantly lower than its number of entries. The first approach is to reformulate the least-squares problem as a nuclear norm minimization problem~(see line~\ref{alg.line.nucnorm} of Algorithm~\ref{alg:2RNN-SL}). The nuclear norm is the tightest convex relaxation of the matrix rank and the resulting optimization problem can be solved using standard convex optimization toolbox~\citep{candes2011tight,recht2010guaranteed}. A second approach is a non-convex optimization algorithm: iterative hard thresholding~(IHT)~\citep{jain2010guaranteed}~(see lines~\ref{alg.iht.start}-\ref{alg.iht.stop} of Algorithm~\ref{alg:2RNN-SL}). This optimization method is iterative and  boils down to a projected gradient descent algorithm: at each iteration, the Hankel tensor is updated by taking a step in the direction opposite to the gradient of the least-squares objective, before being projected onto the manifold of low rank matrices using truncated SVD. More precisely, first the following gradient update is performed:
\begin{align*}
\tenmatgen{\Hten^{(l)}}{l,1} 
&\leftarrow
\tenmatgen{\Hten^{(l)} }{l,1} - \gamma \nabla_{\tenmatgen{\Hten^{(l)} }{l,1}} \norm{\X\tenmatgen{\Ht^{(l)}}{l,1} - \Ymat}_F^2\\
&= 
\tenmatgen{\Hten^{(l)} }{l,1} + \gamma\X^\top(\Ymat - \X\tenmatgen{\Hten^{(l)} }{l,1})
\end{align*}
where $\gamma$ is the learning rate.
Then, a truncated SVD of the matricization $\tenmatgen{\Hten^{(l)}}{\ceil{l/2},l-\ceil{l/2} + 1}$ is performed to obtain a low rank approximation of the Hankel tensor.

Both the nuclear norm minimization and the iterative hard thresholding algorithm only leverages the fact that the matrix rank of $\tenmatgen{\Hten^{(l)}}{\ceil{l/2},l-\ceil{l/2} + 1}$ is small. However, as we have shown in Section~\ref{sec:TT.struct.hankel}, the Hankel tensor $\Hten^{(l)}$ exhibits a stronger structure: it is of low tensor train rank~(which implies that \emph{any} of its matricization is a low rank matrix). We now present three methods leveraging this structure for the recovery of the Hankel tensors from linear measurements. The first optimization algorithm is tensor iterative hard thresholding~(TIHT)~\citep{rauhut2017low} which is the tensor generalization of IHT. Similarly to IHT, TIHT is a projected gradient descent algorithm where the projection step consists in projecting the Hankel tensor onto the manifold of tensors with low tensor train rank~(instead of projecting onto the set of low rank matrices): after the gradient update described above, a low rank tensor train approximation of the Hankel tensor $\Hten^{(l)}$ is computed using the TT-SVD algorithm~\citep{oseledets2011tensor}. 

Even though TIHT leverages the tensor train structure of the Hankel tensors to obtain better sample complexity, its computational complexity remains high since the Hankel tensor $\Hten^{(l)}$ needs to alternatively be converted between its dense form~(for the gradient descent step) and its tensor train decomposition~(for the projection steps). Observe here that the size of these two objects significantly differs: the full Hankel tensor $\Hten^{(l)}$ has size $d^lp$ whereas the number of parameters of its tensor train decomposition is only in $\bigo{ldR^2+pR}$, where $R$ is the rank of the tensor train decomposition. Similarly to the efficient learning algorithm in the tensor train format presented in Section~\ref{sec:spectral.learning.TT}, the recovery of the Hankel tensors can be carried out in the tensor train format without never having to explicitly construct the tensor $\Hten^{(l)}$. We conclude by presenting two optimization methods to recover the Hankel tensors from data directly in the tensor train format. For both methods, the Hankel tensor $\Hten^{(l)}$ is never explicitly constructed but parameterized by the core tensors $\Gt_1,\cdots,\Gt_{l+1}$ of its TT decomposition:
$$\Hten^{(l)} = \TT{\Gten_1,\cdots,\Gten_{l+1}}.$$
Both methods are iterative and will optimize the least-squares objective with respect to each of the core tensors in turn until convergence. The first one is the alternating least-squares algorithm~(ALS), which is one of the workhorse of tensor decomposition algorithms~\citep{Kolda09}. In ALS, at each iteration a least-squares problem is solved in turn for each one of the cores of the TT decomposition:
$$\Gten_i\leftarrow \argmin_{\Gten_i} \norm{\X\tenmatgen{\TT{\Gten_1,\cdots,\Gten_{l+1}}}{l,1} - \Ymat}_F^2\ \ \text{for }i=1,\cdots,l+1.$$
The second one consists in simply using gradient descent to perform a gradient step with respect to each one of the core tensors at each iteration:
$$\Gten_i\leftarrow {\Gten_i} -\gamma\nabla_{\Gten_i}\norm{\X\tenmatgen{\TT{\Gten_1,\cdots,\Gten_{l+1}}}{l,1} - \Ymat}_F^2\ \ \text{for }i=1,\cdots,l+1$$
where $\gamma$ is the learning rate.
Both methods are described in lines~\ref{alg.line.sgdals.start}-\ref{alg.line.sgdals.end} of Algorithm~\ref{alg:2RNN-SL}. Combining these two optimization methods with the spectral learning algorithm in the tensor train format described in Section~\ref{sec:spectral.learning.TT} results in an efficient learning algorithm to estimate a linear 2-RNN from training data, where the Hankel tensors are never explicitly constructed but always manipulated in the tensor train format. 

To conclude this section, we briefly mention that the ALS and gradient descent algorithms can straightforwardly be adapted to perform optimization with respect to mini-batches instead of the whole training dataset. This allows us to further scale the algorithm to large training sets.

\section{Experiments}\label{sec:xp}

In this section, we perform experiments on two toy examples to compare how the choice of the recovery method~(\texttt{LeastSquares}, \texttt{NuclearNorm}, \texttt{IHT}, \texttt{TIHT}, \texttt{ALS} and \texttt{Gradient Descent}) affects the sample efficiency of Algorithm~\ref{alg:2RNN-SL}, and the corresponding computation time. We  also report the performance obtained by refining the solutions returned by our algorithm~(with both TIHT and ALS recovery methods) using stochastic gradient descent~(\texttt{TIHT+SGD}, \texttt{ALS+SGD}). In addition, we perform experiments on a real world dataset of wind speed data from TUDelft, which is used in~\citep{lin2016short}. For the real world data, we include the original results for competitive approaches from~\citep{lin2016short}. 

\subsection{Synthetic data}

We perform experiments on two toy problems: recovering a random 2-RNN from data and a simple addition task. For the random 2-RNN problem, we randomly generate a linear 2-RNN with $5$ units computing a function $f:\R^3\to\R^2$ by 
drawing the entries of all parameters  $(\h_0,\Aten,\vvsinf)$  independently  from a normal distribution
$\Ncal(0,0.2)$.
The training data consists of $3$ independently drawn sets
$D_l = \{((\x^{(i)}_1,\x_2^{(i)},\cdots,\x_l^{(i)}),\y^{(i)}) \}_{i=1}^{N_l}\subset(\R^d)^l\times \R^p$ for $l\in\{L,2L,2L+1\}$ with $L=2$,
where  each $\x^{(i)}_j\sim\Ncal(\vec{0},\I)$ and where the outputs can be noisy, \ie 
$\y^{(i)} = f(\x^{(i)}_1,\x_2^{(i)},\cdots,\x_l^{(i)}) + \vecs{\xi}^{(i)}$ where $\vecs{\xi}^{(i)}\sim\Ncal(0,\sigma^2)$ for some noise variance parameter $\sigma^2$. 
For the addition problem, the goal is to learn a simple arithmetic function computing the sum of the running differences between the two components of a sequence
of $2$-dimensional vectors, \ie $f(\x_1,\cdots,\x_k) = \sum_{i=1}^k \v^\top\x_i$ where $\v^\top = (-1\ \ 1)$. The $3$ training datasets are generated 
using the same process as above and a constant entry equal to one is added to all the input vectors to encode a bias 
term~(one can check that the resulting function can be computed by a linear 2-RNN with $2$ hidden units). 

\begin{figure*}[t]
\begin{center}
\hspace*{-1cm}\includegraphics[width=1.\textwidth]{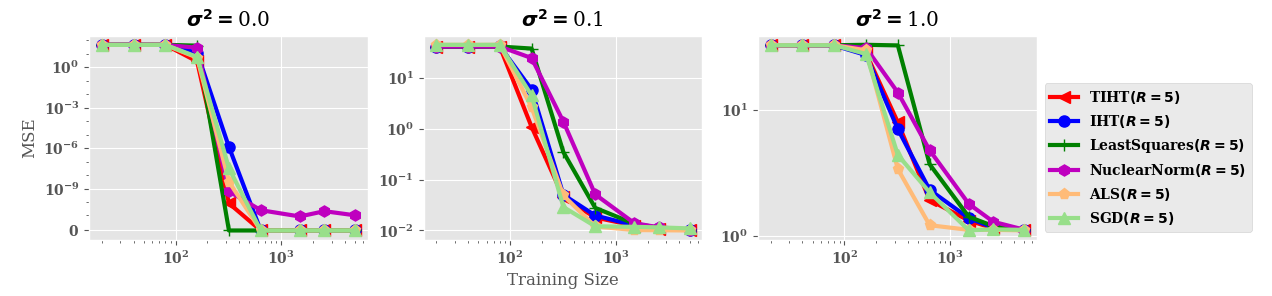}
\end{center}
\caption{~Average MSE as a function of the training set size for learning a random linear 2-RNN with different values of output noise.}
\label{fig:RandomRNN}
\end{figure*}
\begin{figure*}[h]
\begin{center}
\hspace*{-1cm}\includegraphics[width=1.\textwidth]{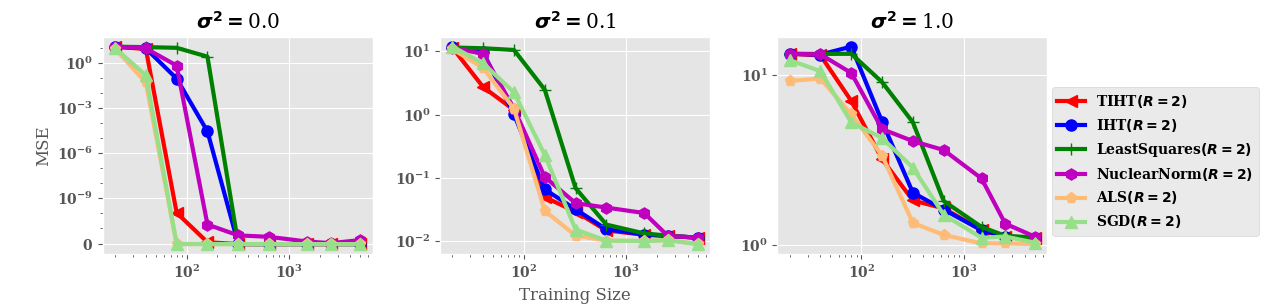}
\end{center}
\caption{~Average MSE as a function of the training set size for  learning a simple arithmetic function with different values of output noise.}
\label{fig:Addition}
\end{figure*}

We run the experiments for different sizes of training data ranging from $N=20$ to $N=5,000$~(we set $N_L=N_{2L}=N_{2L+1}=N$) 
and we compare the different methods in terms of mean squared error~(MSE)
 on a test set of $1,000$ sequences of length $6$ generated in the same
way as the training data~(note that the training data only contains sequences of length up to $5$). The IHT/TIHT methods sometimes returned aberrant models~(due to numerical instabilities), we used the following scheme to circumvent this issue: 
when the training MSE of the hypothesis was greater than the one of the zero function,
the zero function was returned instead (we applied this scheme to all other methods in the experiments). For the gradient descent approach, we use the autograd method from Pytorch with the Adam~\citep{kingma2014adam} optimizer with learning rate $0.001$. 

\subsubsection{Results}
The results are reported in Figure~\ref{fig:RandomRNN} and~\ref{fig:Addition} where we see that all recovery methods lead to consistent estimates of the
target function given enough training data. This is the case even in the presence of noise~(in which case more samples are needed to achieve the same accuracy, as expected).
We can also see that \texttt{TIHT} and \texttt{ALS} tend to be overall more sample efficient than the other methods~(especially with noisy data), showing
that taking the low rank structure of the Hankel tensors into account is profitable. Moreover, \texttt{TIHT} tends to perform better than its matrix counterpart, 
confirming our intuition that leveraging the tensor train 
structure is beneficial.

We also found that using gradient descent to refine the learned 2-RNN model often leads to a performance boost. In Figure~\ref{fig:fine_tune_RandomRNN} and Figure~\ref{fig:fine_tune_Addition} we show the advantage MSE obtained by fine tuning the learned 2-RNN using gradient descent. We use Pytorch to implement the fine-tuning process with the Adam optimizer with a learning rate of 0.0001.  Fine tuning helps the model to converge to the optimal solution with less  data, resulting in a more sample efficient approach.
Lastly, we briefly mention than on these two tasks, previous experiments showed that both non-linear and linear recurrent neural network architectures trained with the back-propagation algorithm performed significantly worse than the spectral learning based learning algorithm we propose~(see~\cite{rabusseau2019connecting}).

\begin{figure*}
\begin{center}
\hspace*{-1cm}\includegraphics[width=1.\textwidth]{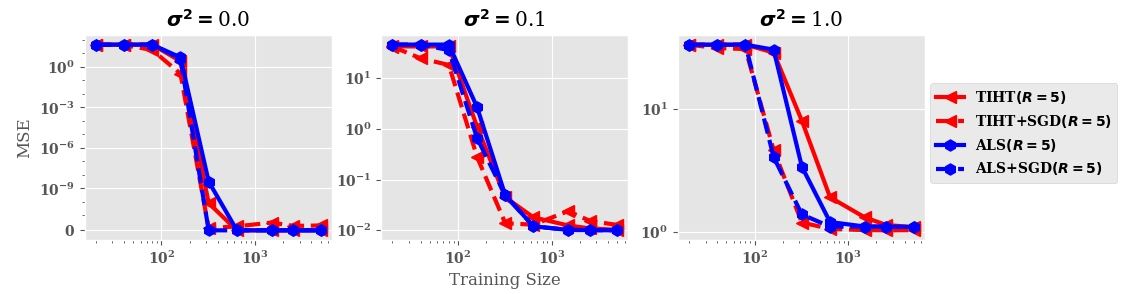}
\end{center}
\caption{~Performance comparison between vanilla methods and fine-tuned methods on Random 2-RNN problem.}
\label{fig:fine_tune_RandomRNN}
\end{figure*}

\begin{figure*}
\begin{center}
\hspace*{-1cm}\includegraphics[width=1.\textwidth]{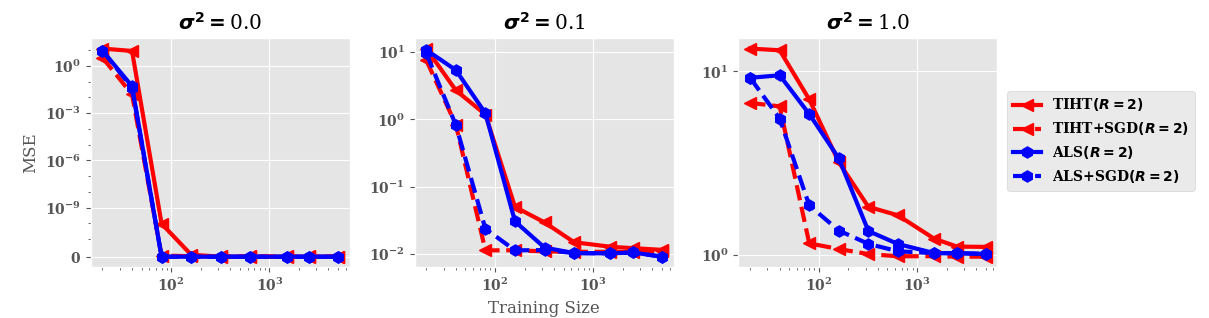}
\end{center}
\caption{~Performance comparison between vanilla methods and fine-tuned methods on Addition problem.}
\label{fig:fine_tune_Addition}
\end{figure*}

\subsubsection{Running time analysis}

By directly recovering the Hankel tensor in its tensor train form, \texttt{ALS} and \texttt{SGD} significantly reduces the computation time needed to recover the Hankel tensor. In Figure~\ref{fig:runtime}, we report the computation time of the different Hankel recovery methods for Hankel tensors with various length~($L$). The experiment is performed with 1,000 examples for the addition problem and all  iterative methods (excluding \texttt{OLS}) are stopped when reaching the same fixed training accuracy. In the figure, there is a clear reduction in computation time for both \texttt{ALS} and \texttt{SGD} compared to other methods, which is expected. More specifically, these methods have much smaller computation time growth rate with respect to the length $L$ compared to the matrix-based methods. This is especially beneficial when dealing with data that exhibits long term dependencies of the input variables. In comparison to the Hankel tensor recovery time, the spectral learning step takes significantly less time, typically within a second. However, one important note is that if the length $L$ gets larger, directly performing spectral learning on the matrix form of the Hankel tensor may not be possible due to the curse of dimensionality. Therefore, under this circumstance, one should directly perform the spectral learning algorithm in its tensor train form as described in Section~\ref{sec:spectral.learning.TT}.

 To demonstrate the benefits of performing the spectral learning algorithm in the TT format~(as described in Section~\ref{sec:spectral.learning.TT}), we perform an additional experiment showing that 
leveraging the TT format allows one to save significant amount of computation time and memory resources in the spectral learning phase, especially when the corresponding Hankel tensor is large (i.e. large length and input dimension). 
In Figure~\ref{fig:spectral time} we compare the running time of the spectral learning phase (after recovering the Hankel tensors) in the  matrix and TT formats, where the latter leverages the TT structure in the spectral learning routine. We randomly generate 100,000 input-output examples using a Random 2-RNN with 3 states, input dimension 5 and output dimension 1. We use ALS to  recover the Hankel tensors in the TT format and compare the running time of the spectral learning in the TT format with the time needed to perform the classical spectral learning algorithm after reshaping the Hankel tensors in matrices~(note that the time needed to convert the TT Hankel tensors into the corresponding Hankel matrices is not counted towards the matrix spectral learning time). 
In Figure~\ref{fig:spectral time}, we report the time needed to recover the Hankel tensors from data~(\texttt{Hankel\_ALS}) and the time to recover the WFA in both the matrix and TT formats. 
One can observe that although classic matrix-based spectral learning is significantly faster than the TT-based one when the length  is relatively small, the running time of the matrix method grows exponentially with the length while the one of the TT method is linear. For example, when the length equals to 12, TT spectral learning is more than 1,000 times faster than the classic spectral learning. 
This computation time gap significantly shows the benefit of leveraging TT format in the spectral learning phase. One remark is that other types of Hankel tensor recovery methods that we mentioned (i.e. TIHT, IHT, LeastSquares and NuclearNorm) fail to scale in this setup, due to excessive memory required by these algorithms in preparing the training data. 

\begin{figure*}[t]

     \centering
     \begin{subfigure}[t]{0.48\textwidth}
         \centering
         \includegraphics[width=\textwidth]{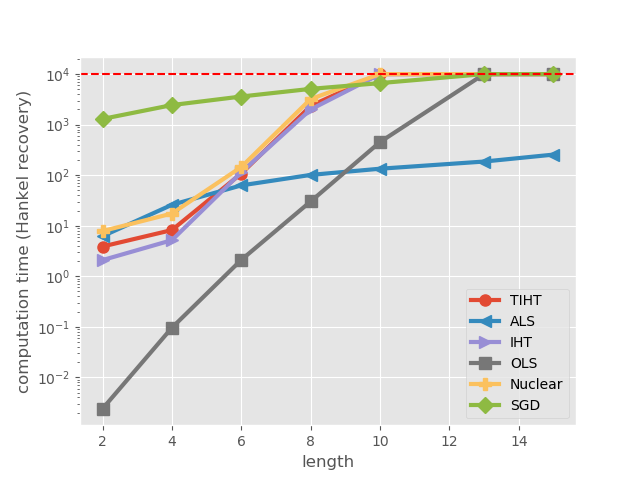}
         \caption{Computation time comparison between different Hankel recovery methods on addition problem with 1,000 data. Computation time is capped at 10,000 seconds for all methods (the red dashed line).}
         
         \label{fig:runtime}
     \end{subfigure}\hspace{0.02\textwidth}
     \begin{subfigure}[t]{0.48\textwidth}
         \centering
         
         \includegraphics[width=\textwidth]{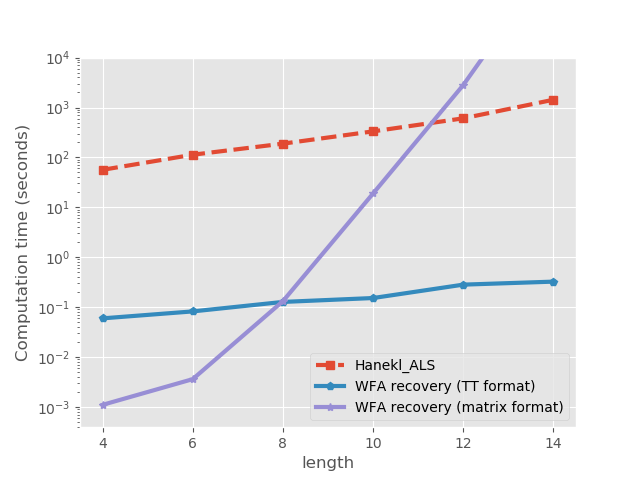}
         \caption{Computation time comparison between TT-spectral learning and classic matrix based spectral learning on Random 2-RNN problem with 100,000 examples. The ground truth 2-RNN has 3 states, with input dimension 5 and output dimension 1.}
         \label{fig:spectral time}
     \end{subfigure}

\caption{Running time comparison}
\label{fig:running time comparison}
\end{figure*}

In addition, directly recovering the Hankel tensors and performing spectral learning in the TT format  also helps drastically reduce the memory resources. As an illustration, we compare the size of the Hankel matrix in  the TT format and the matrix format in Table~\ref{tab:Hankel size}. As one can see the size of the matrix version of the Hankel grows exponentially w.r.t the length while the TT Hankel size grows linearly. This also echoes with the computation time for these two methods.

\begin{table}[h]
\centering
\begin{tabular}{lllllll}
\hline
\multicolumn{1}{l|}{Length}                & 4        & 6        & 8        & 10       & 12       & 14       \\ \hline
\multicolumn{1}{l|}{TT Hankel Size (GB)}   & 2.68e-06 & 4.69e-06 & 6.70e-06 & 8.71e-06 & 1.07e-05 & 1.27e-05 \\
\multicolumn{1}{l|}{Matrix Hankel Size (GB)} & 1.40e-05 & 3.49e-04 & 8.73e-03 & 2.20e-01 & 5.40e-01 & 136 \\\hline
\end{tabular}
\caption{Memory size of the Hankel tensor $\Ht^{(\ell)}$ for the random 2-RNN problem~(see Figure~\ref{fig:spectral time}) in both TT and matrix formats. }
\label{tab:Hankel size}
\end{table}


\subsection{Real world data}
In addition to the synthetic data experiments presented above, we conduct experiments on the wind speed data from TUDelft \footnote{http://weather.tudelft.nl/csv/}. For this experiment, to compare with existing results, we specifically use the data from Rijnhaven station as described in~\citet{lin2016short}, which proposed a regression automata model and performed various experiments on the wind speed dataset. The data contains wind speed and related information at the Rijnhaven station from  2013-04-22 at 14:55:00 to 2018-10-20 at 11:40:00 and was collected every five minutes. To compare with the results in~\citep{lin2016short}, we strictly followed the data preprocessing procedure described in the paper. We use the data from 2013-04-23 to 2015-10-12 as training data and the rest as our testing data.  The paper uses SAX as a preprocessing method to discretize the data. However, as there is no need to discretize data for our algorithm, we did not perform this procedure. For our method, we set the length $L = 3$ and we use a window size of 6 to predict the future values at test time. We calculate hourly averages of the wind speed, and predict one/three/six hour(s) ahead, as in~\citep{lin2016short}. In this experiment, our model only predicts the next hour from the past 6 observations. To make $k$-hour-ahead prediction, we use the forecast of the model itself  as input and bootstrap from it. For our methods we use a linear 2-RNN with 10 states. Averages over 5 runs of this experiment for one-hour-ahead, three-hour-ahead, six-hour-ahead prediction error can be found in Table~\ref{one_hour}, \ref{three_hours} and \ref{six_hours}. The results for RA, RNN and persistence are taken directly from~\citep{lin2016short}. 
 
The results of this experiment are presented in Table~\ref{one_hour}-\ref{six_hours} where we can see that while TIHT+SGD performs slightly worse than ARIMA and RA for one-hour-ahead prediction, it outperforms all other methods for three-hours and six-hours ahead predictions~(and the superiority w.r.t. other methods increases as the prediction horizon gets longer). One important note is that although ALS and ALS+SGD is slightly under-performing compared to TIHT and TIHT+SGD, the computation time has been significantly reduced for ALS  by a factor of 5~(TIHT takes 3,542 seconds while ALS takes 804 seconds).

\begin{table}[h]
\caption{~One-hour-ahead Speed Prediction Performance Comparisons}
\centering
\begin{tabular}{c|cccccccc}
Method & TIHT     & \begin{tabular}[c]{@{}c@{}}TIHT \\ +SGD\end{tabular} &ALS&\begin{tabular}[c]{@{}c@{}}ALS \\ +SGD\end{tabular}& \begin{tabular}[c]{@{}c@{}}Regression \\ Automata\end{tabular} & ARIMA & RNN & Persistence\\ \hline
RMSE   & 0.573  & 0.519&0.586& 0.522   & 0.500    & \textbf{0.496 }& 0.606 & 0.508                                                    \\
MAPE   & 21.35  & 18.79&22.12& 19.01    & \textbf{18.58}   & 18.74 & 24.48 & 18.61                                                    \\
MAE    & 0.412  & 0.376&0.423&0.388    & 0.363    & \textbf{0.361} & 0.471 & 0.367                                                    
\end{tabular}%
\label{one_hour}
\end{table}
\begin{table}[h]
\caption{~Three-hour-ahead Speed Prediction Performance Comparisons}
\centering
\begin{tabular}{c|cccccccc}
Method & TIHT     & \begin{tabular}[c]{@{}c@{}}TIHT \\ +SGD\end{tabular} &ALS&\begin{tabular}[c]{@{}c@{}}ALS \\ +SGD\end{tabular}& \begin{tabular}[c]{@{}c@{}}Regression \\ Automata\end{tabular} & ARIMA & RNN & Persistence\\ \hline
RMSE   & 0.868  & \textbf{0.854}&0.875&0.864    & 0.872  & 0.882 & 1.002 & 0.893                                                       \\
MAPE   & 33.98  & \textbf{31.70}&34.67&32.13    & 32.52 & 33.165 & 37.24 & 33.29                                                      \\
MAE    & 0.632  & \textbf{0.624}&0.648&0.628    & 0.632 &0.642 & 0.764 & 0.649                                                       
\end{tabular}
\label{three_hours}
\end{table}
\begin{table}[h]
\caption{~Six-hour-ahead Speed Prediction Performance Comparisons}
\centering
\begin{tabular}{c|cccccccc}
Method & TIHT     & \begin{tabular}[c]{@{}c@{}}TIHT \\ +SGD\end{tabular} &ALS&\begin{tabular}[c]{@{}c@{}}ALS \\ +SGD\end{tabular}& \begin{tabular}[c]{@{}c@{}}Regression \\ Automata\end{tabular} & ARIMA & RNN & Persistence\\ \hline
RMSE   & 1.234  & \textbf{1.145} &1.283&1.128   & 1.205  & 1.227 & 1.261 & 1.234                                                      \\
MAPE   & 49.08  & \textbf{44.88}&47.65&45.03    & 46.809  &48.02 & 47.03 & 48.11                                                      \\
MAE    & 0.940  & \textbf{0.865}&0.932&0.869    & 0.898  & 0.919 & 0.944 & 0.923            
\end{tabular}
\label{six_hours}
\end{table}





\section{Conclusion and Future Directions}

We proposed the first provable learning algorithm for second-order RNN with linear activation functions:
we showed that linear 2-RNN are a natural extension of vv-WFA to the setting of input sequences of \emph{continuous vectors}~(rather than
discrete symbol) and we extended the vv-WFA spectral learning
algorithm to this setting. We also presented novel connections between WFA and tensor networks, showing that the computation of a WFA is intrinsically linked with the tensor train decomposition. We leveraged this connection to  adapt the standard spectral learning algorithm to the tensor train format, allowing one to scale up the spectral algorithm to exponentially large sub-blocks of the Hankel matrix.

We believe that the results presented in this paper open a number of exciting and promising research directions on both the
theoretical and practical perspectives. We first plan to use the spectral learning estimate as a starting point for
gradient based methods to train non-linear 2-RNN. More precisely, linear 2-RNN can be thought of as 2-RNN using LeakyRelu activation functions with negative slope $1$, therefore one could use 
a linear 2-RNN as initialization before gradually reducing the negative slope parameter during training. The extension of the spectral method to linear 2-RNN also
opens the door to scaling up the classical spectral algorithm to problems with large discrete alphabets~(which is a known caveat of the spectral algorithm for WFA) since
it allows one to use low dimensional embeddings of large vocabularies~(using \eg word2vec or latent semantic analysis). From the theoretical perspective, we plan on 
deriving learning guarantees for  linear 2-RNN in the noisy setting~(\eg using the PAC learnability framework). Even though it is intuitive that such guarantees should hold~(given
the continuity of all operations used in our algorithm), we believe that such an analysis may entail results of independent interest. In particular, analogously to the
matrix case  studied in~\citep{cai2015rop}, obtaining optimal convergence rates for the recovery of the low TT-rank Hankel tensors from rank one measurements is an interesting
direction; such a result could for example allow one to improve the generalization bounds provided in~\citep{balle2012spectral} for spectral learning of general WFA. Lastly, establishing other equivalence results between classical classes of formal languages and functions computed by recurrent architectures is a worthwhile endeavor; such equivalence results give a novel light on classical models from theoretical computer science and linguistics while at the same time sparkling original perspectives on modern machine learning architectures. A first direction could be to establish connections between weighted tree automata and  tree-structured neural models such as recursive tensor neural networks~\citep{socher2013recursive,socher2013reasoning}.

\subsection*{\textbf{Acknowledgements}}
This research is supported by the Canadian Institute for Advanced Research (CIFAR AI chair program) and Fonds de Recherche du Québec – Nature et technologies (no. 271273). We also thank Compute Canada and Calcul Québec for the computing resources.

\bibliographystyle{spbasic}      
\bibliography{main}   

\end{document}